\documentclass{llncs}

\usepackage{amsmath,amssymb,amsfonts}
\usepackage{fullpage}
\usepackage{booktabs}
\usepackage{verbatim}
\usepackage{wasysym}
\usepackage{paralist}
\usepackage{multirow}
\usepackage{algorithmicx}
\usepackage[noend]{algpseudocode}
\usepackage{algorithm}
\usepackage{algpseudocode}

\newcommand{\game}{G}
\newcommand{\states}{S}
\newcommand{\informationstates}{\mathcal{H}}
\newcommand{\act}{\mathcal{A}}
\newcommand{\trans}{\delta}
\newcommand{\obs}{\mathcal{Z}}
\newcommand{\obsmap}{\mathcal{O}}
\newcommand{\distr}{\mathcal{D}}
\newcommand{\powset}{\mathcal{P}}
\newcommand{\cost}{\mathsf{c}}

\newcommand{\almost}{\mathsf{Almost}}

\newcommand{\set}[1]{\{#1\}}
\newcommand{\straa}{\sigma}
\newcommand{\restr}{\!\!\upharpoonright}

\newcommand{\supp}{\mathrm{Supp}}
\newcommand{\prb}{\mathbb{P}}
\newcommand{\good}{\mathsf{good}}
\newcommand{\bad}{\mathsf{bad}}
\newcommand{\targetstate}{\mathsf{target}}

\newcommand{\last}{\mathsf{Last}}
\newcommand{\belief}{\mathcal{B}}
\newcommand{\beliefset}{\mathsf{Belief}}
\newcommand{\beliefsetwin}{\mathsf{Belief}_{\mathsf{Win}}}
\newcommand{\wh}{\widehat}

\newcommand{\allow}{{\textsf{Allow}}}

\newcommand{\reach}{{\mathsf{Reach}}}
\newcommand{\sumcost}{{\mathsf{Total}}}

\newcommand{\Cone}{\mathsf{Cone}}

\newcommand{\update}{{\mathsf{Update}}}

\newcommand{\val}{\mathsf{Val}}

\newcommand{\Reach}{\mathsf{Reach}}

\newcommand{\ov}{\overline}

\newcommand{\target}{T}

\newcommand{\wb}{\overline}

\newcommand{\reals}{\mathbb{R}}
\newcommand{\integer}{\mathbb{Z}}
\newcommand{\nat}{\mathbb{N}}
\newcommand{\expect}{\mathbb{E}}

\newcommand{\optCost}{\mathsf{optCost}}
\newcommand{\upper}{\mathcal{U}_{\allow}}

\newcommand{\cale}{\mathcal{E}}
\newcommand{\Exp}{\mathbb{E}}

\newcommand{\automaton}{\mathsf{P}}

\usepackage{tikz}
\usetikzlibrary{fit}
\usetikzlibrary{backgrounds}
\usetikzlibrary{automata}
\usetikzlibrary{shapes}
\usetikzlibrary{matrix}
\usetikzlibrary{fit}
\usetikzlibrary{positioning}
\usetikzlibrary{calc}

\tikzstyle{Player1}=[circle, thick, minimum size=1.0cm, inner sep=0cm,draw=black]
\tikzstyle{State}=[circle, thick, minimum size=0.6cm, inner sep=0cm,draw=black]
\tikzstyle{Final}=[circle, accepting, thick, minimum size=0.6cm, inner sep=0cm,draw=black]
\tikzstyle{RState}=[circle, very thick, minimum size=0.8cm, inner sep=0cm,draw=red]
\title{Optimal Cost Almost-sure Reachability in POMDPs\thanks{The research was partly supported by Austrian Science Fund (FWF) Grant No P23499-N23, 
FWF NFN Grant No S11407-N23 (RiSE), ERC Start grant (279307: Graph Games), and 
Microsoft faculty fellows award.}\\ (Full Version)}
\author{Krishnendu Chatterjee$^1$ \and Martin Chmel\'ik$^1$ \and Raghav Gupta$^2$ \and Ayush Kanodia$^2$}
\institute{
$^1$ IST Austria, Klosterneuburg, Austria  \\
$^2$ IIT Bombay, India  \\}

\begin{document}

\maketitle

\begin{abstract}
We consider partially observable Markov decision processes (POMDPs)
with a set of target states and every transition is associated with an 
integer cost.
The optimization objective we study asks to minimize the expected total 
cost till the target set is reached, while ensuring that the target set
is reached almost-surely (with probability~1).
We show that for integer costs approximating the optimal cost is undecidable.
For positive costs, our results are as follows: (i)~we establish matching 
lower and upper bounds for the optimal cost and the bound is double exponential; 
(ii)~we show that the problem of approximating the optimal cost is 
decidable and present approximation algorithms developing on the existing 
algorithms for POMDPs with finite-horizon objectives. 
While the worst-case running time of our algorithm is double exponential,
we also present efficient stopping criteria for the algorithm and show 
experimentally that it performs well in many examples of interest.
\end{abstract}

\section{Introduction}

\noindent{\bf Partially observable Markov decision processes (POMDPs).}
\emph{Markov decision processes (MDPs)} are standard models for 
probabilistic systems that exhibit both probabilistic as well as 
nondeterministic behavior~\cite{Howard}.
MDPs are widely used to model and solve control problems for stochastic 
systems~\cite{FV97,Puterman}: 
nondeterminism represents the freedom of the controller to choose a 
control action, while the probabilistic component of the behavior describes the 
system response to control actions. 
In \emph{perfect-observation (or perfect-information) MDPs} the 
controller observes the current state of the system precisely to choose the 
next control actions, whereas in \emph{partially observable MDPs (POMDPs)} the 
state space is partitioned according to observations that the controller can 
observe, i.e., given the current state, the controller can only view the observation 
of the state (the partition the state belongs to), but not the precise 
state~\cite{PT87}.
POMDPs provide the appropriate model to study a wide variety of applications 
such as in computational biology~\cite{Bio-Book}, 
speech processing~\cite{Mohri97}, image processing~\cite{IM-Book}, %%software verification~\cite{CCHRS11}, 
robot planning~\cite{KGFP09,kaelbling1998planning}, reinforcement learning~\cite{LearningSurvey}, 
to name a few. 
POMDPs also subsume many other powerful computational models such as 
probabilistic finite automata (PFA)~\cite{Rabin63,PazBook} 
(since probabilistic finite automata (aka blind POMDPs) are a special case of 
POMDPs with a single observation).

\smallskip\noindent{\bf Classical optimization objectives.}
In stochastic optimization problems related to POMDPs, 
the transitions in the POMDPs are associated with integer costs, 
and the two classical objectives that have been widely studied are 
\emph{finite-horizon} and \emph{discounted-sum} objectives~\cite{FV97,Puterman,PT87}.
For finite-horizon objectives, a finite length $k$ is given and the goal is 
to minimize the expected total cost for $k$ steps.
In {discounted-sum} objectives, the cost in the $j$-th step is multiplied
by $\gamma^j$, for $0< \gamma < 1$, and the goal is to minimize the expected 
total discounted cost over the infinite horizon.

\smallskip\noindent{\bf Reachability and total-cost.} 
In this work we consider a different optimization objective 
for POMDPs. 
We consider POMDPs with a set of target states, and the 
optimization objective is to minimize the expected total cost 
till the target set is reached. 
First, note that the objective is not the discounted sum, but the
total sum without discounts. 
Second, the objective is not a finite-horizon objective, as there 
is no bound apriori known to reach the target set, and along 
different paths the target set can be reached at different time 
points. 
The objective we consider is very relevant in many control 
applications such as in robot planning: for example, the robot 
has a target or goal;  and the objective is to minimize the number 
of steps to reach the target, or every transition is associated with 
energy consumption and the objective is to reach the target with minimal
energy consumption.

\smallskip\noindent{\bf Our contributions.}
In this work we study POMDPs with a set of target states, and costs
in every transition, and the goal is to minimize the expected total 
cost till the target set is reached, while ensuring that the target set 
is reached almost-surely (with probability~1).
Our results are as follows:
\begin{compactenum}
\item \emph{(Integer costs).} 
We first show that if the transition costs are integers, then approximating 
the optimal cost is undecidable.

\item \emph{(Positive integer costs).} Since the problem is undecidable for
integer costs, we next consider that costs are positive integers.
We first remark that if the costs are positive, and there is a positive 
probability not to reach the target set, then the expected total cost 
is infinite. 
Hence the expected total cost is not infinite only by ensuring that the 
target is reached almost-surely.
First we establish a double-exponential lower and upper bound for the 
expected optimal cost.
We show that the approximation problem is decidable, and present approximation
algorithms using the well-known algorithms for finite-horizon objectives.

\item \emph{(Implementation).} 
Though we establish that in the worst-case the algorithm requires double-exponential time,
we also present efficient stopping criteria for the algorithm, and 
experimentally show that the algorithm is efficient in several practical examples.
We have implemented our approximation algorithms developing on the existing 
implementations for finite-horizon objectives, and present experimental results 
on a number of well-known examples of POMDPs.

\end{compactenum}

\smallskip\noindent{\em Comparison with Goal-POMDPs.} 
While there are several works for discounted POMDPs~\cite{KHL08,SS04,PGT03}, as mentioned above
the problem we consider is different from discounted POMDPs. 
The most closely related works are Goal-MDPs and POMDPs~\cite{BG09,KMWG11}. The key differences are as follows:
(a)~our results for approximation apply to all POMDPs with positive integer costs,
whereas the solution for Goal-POMDPs apply to a strict subclass of POMDPs 
(see Remark~\ref{rem:almostsure}); and (b)~we present asymptotically tight (double exponential) 
theoretical bounds on the expected optimal costs.

\newcommand{\C}{\mathcal{C}}
\newcommand{\initd}{\lambda}

\section{Definitions}
We present the definitions of POMDPs, strategies, objectives, 
and other basic notions required for our results.
Throughout this work, we follow standard notations from~\cite{Puterman,LittmanThesis}.

\smallskip\noindent{\bf Notations.}
Given a finite set $X$, we denote by $\powset(X)$ the set of subsets of $X$,
i.e., $\powset(X)$ is the power set of $X$.
A probability distribution $f$ on $X$ is a function $f:X \to [0,1]$ such 
that $\sum_{x\in X} f(x)=1$, and we denote by  $\distr(X)$ the set of 
all probability distributions on $X$. For $f \in \distr(X)$ we denote by $\supp(f)=\set{x\in X \mid f(x)>0}$
the support of $f$.

\smallskip\noindent{\bf POMDPs.}
A \emph{Partially Observable Markov Decision Process (POMDP)} is a 
tuple $\game=(S,\act,\trans,\obs,\obsmap,\initd_0)$ where:
(i)~$S$ is a finite set of states; (ii)~$\act$ is a finite alphabet of \emph{actions};
 (iii)~$\trans:S\times\act \rightarrow \distr(S)$ is a 
 \emph{probabilistic transition function} that given a state $s$ and an
 action $a \in \act$ gives the probability distribution over the successor 
 states, i.e., $\trans(s,a)(s')$ denotes the transition probability from 
 $s$ to $s'$ given action $a$; 
 %%($\distr(S)$ denotes the set of  probability distributions over $S$). 
 (iv)~$\obs$ is a finite set of \emph{observations}; 
 (v)~$\obsmap:S\rightarrow \obs$ is an \emph{observation function} that 
  maps every state to an observation; and 
 (vi)~$\initd_0$ is a probability distribution for the initial state,
and for all $s,s'\in \supp(\initd_0)$ we require that $\obsmap(s)=\obsmap(s')$.
If the initial distribution is Dirac, we often write $\initd_0$ as $s_0$ 
where $s_0$ is the unique starting (or initial) state.
%%over the initial state. 
% \end{definition}
%%\noindent
Given $s,s'\in S$ and $a\in\act$, we also write $\trans(s'|s,a)$ for 
$\trans(s,a)(s')$.
A state $s$ is \emph{absorbing} if for all actions $a$ we have 
$\trans(s,a)(s)=1$ (i.e., $s$ is never left from $s$).
For an observation $z$, 
we denote by $\obsmap^{-1}(z)=\set{s \in S \mid \obsmap(s)=z}$
the set of states with observation $z$.
For a set $U \subseteq S$ of states and $Z \subseteq \obs$ of observations 
we denote 
$\obsmap(U)=\set{z \in \obs \mid \exists s \in U. \ \obsmap(s)=z}$ 
and $\obsmap^{-1}(Z)= \bigcup_{z \in Z} \obsmap^{-1}(z)$.
%\smallskip\noindent{\bf Perfect-observation MDPs.} 
A POMDP is a \emph{perfect-observation (or perfect-information) MDP} 
if each state has a unique observation. 

\begin{comment}
\begin{remark}
 For technical convenience we assume that there is a unique initial 
 state $s_0$ and we will also assume that the initial state has a unique 
 observation, i.e., $|\obsmap^{-1}(\obsmap(s_0))|=1$.
 In general there is an initial distribution $\alpha$ over initial states
 that all have the same observation, i.e., $\supp(\alpha) \subseteq 
 \obsmap^{-1}(o)$, for some $o \in \obs$.
 However, this can be modeled easily by adding a new initial state 
 $s_{\mathit{new}}$ with a unique observation such that in the first step gives 
 the desired initial probability distribution $\alpha$, 
 i.e., $\trans(s_{\mathit{new}},a)=\alpha$ for all actions $a \in \act$.
 Hence for simplicity we assume there is a unique initial state $s_0$ with a
 unique observation.
 \end{remark}
\end{comment}

\smallskip\noindent{\bf Plays and cones.} 
A \emph{play} (or a path) in a POMDP is an infinite sequence $(s_0,a_0,s_1,a_1,s_2,a_2,\ldots)$ 
of states and actions such that for all $i \geq 0$ we have $\trans(s_i,a_i)(s_{i+1})>0$ and
$s_0 \in \supp(\initd_0)$.  
We write $\Omega$ for the set of all plays.
For a finite prefix $w \in (S\cdot A)^* \cdot S$ of a play, we denote by 
$\Cone(w)$ the set of plays with $w$ as the prefix (i.e., the cone or cylinder 
of the prefix $w$), and denote by $\last(w)$ the last state of $w$.

\smallskip\noindent{\bf Belief-support and belief-support updates.}
For a finite prefix $w=(s_0,a_0,s_1,a_1,\ldots,s_n)$ we denote by 
$\obsmap(w)=(\obsmap(s_0),a_0,\obsmap(s_1),a_1,\ldots,\obsmap(s_n))$ 
the observation and action sequence associated with $w$.
For a finite sequence $\rho=(z_0,a_0,z_1,a_1,\ldots,z_n)$ of observations and 
actions, the \emph{belief-support} $\belief(\rho)$ after the prefix $\rho$ is 
the set of states in which a finite prefix of a play can be after the sequence 
$\rho$ of observations and actions, 
i.e., $$\belief(\rho)=\set{s_n=\last(w) \mid w=(s_0,a_0,s_1,a_1,\ldots,s_n), w \mbox{
is a prefix of a play, and for } \\ \text{all } 0\leq i \leq n. \; \obsmap(s_i)=z_i}.$$ 
The belief-support updates associated with finite-prefixes are as follows:
for prefixes $w$ and $w'=w \cdot a \cdot s$ the belief-support update is 
defined inductively as 
$\belief(\obsmap(w')) = 
\left(\bigcup_{s_1 \in \belief(\obsmap(w))} \supp(\trans(s_1,a)) \right)
\cap \obsmap^{-1}(\obsmap(s))$,
%%%% CUT FOR SHORT VERSION:
 i.e., the set
 $\left(\bigcup_{s_1 \in \belief(\obsmap(w))} \supp(\trans(s_1,a)) \right)$
 denotes the possible successors given the belief-supprt 
 $\belief(\obsmap(w))$ and action $a$, and then the intersection with the set 
 of states with the current observation $\obsmap(s)$ gives the new 
 belief-support set.

\smallskip\noindent{\bf Strategies (or policies).}
A \emph{strategy (or a policy)} is a recipe to extend prefixes of plays and 
is a function $\sigma: (S\cdot A)^* \cdot S \to \distr(A)$ that given a finite 
history (i.e., a finite prefix of a play) selects a probability distribution 
over the actions.
Since we consider POMDPs, strategies are \emph{observation-based}, i.e., 
for all histories $w=(s_0,a_0,s_1,a_1,\ldots,a_{n-1},s_n)$ and 
$w'=(s_0',a_0,s_1',a_1,\ldots,a_{n-1},s_n')$ such that for all 
$0\leq i \leq n$ we have $\obsmap(s_i)=\obsmap(s_i')$ (i.e., $\obsmap(w) = \obsmap(w')$), we must have 
$\sigma(w)=\sigma(w')$.
In other words, if the observation sequence is the same, then the strategy 
cannot distinguish between the prefixes and must play the same. 
A strategy $\straa$ is \emph{belief-support based stationary} if it depends 
only on the current belief-support, i.e., 
whenever for two histories $w$ and $w'$, we have
$\belief(\obsmap(w)) = \belief(\obsmap(w'))$, then $\straa(w) = \straa(w')$.
% We now present an equivalent definition of observation-based strategies  
% such that the memory of the strategy is explicitly specified, and 
% will be required to present finite-memory strategies.

%%%% CUT FOR SHORT VERSION
\smallskip\noindent{\bf Strategies with memory and finite-memory strategies}
 A \emph{strategy} with memory is a tuple $\sigma=(\sigma_u,\sigma_n,M,m_0)$ 
 where:(i)~\emph{(Memory set).} $M$ is a denumerable set (finite or infinite) of memory elements (or memory states).
 (ii)~\emph{(Action selection function).} The function $\sigma_n:M\rightarrow \distr(\act)$ is the 
 	\emph{action selection function} that given the current memory 
 	state gives the probability distribution over actions.
 (iii)~\emph{(Memory update function).} The function $\sigma_u:M\times\obs\times\act\rightarrow \distr(M)$ 
  is the \emph{memory update function} that given the current memory state, 
  the current observation and action, updates the memory state probabilistically.
 (iv)~\emph{(Initial memory).} The memory state $m_0\in M$ is the initial memory state.
 A strategy is a \emph{finite-memory} strategy if the set $M$ of memory elements is finite.
 A strategy is \emph{pure (or deterministic)} if the memory update function 
 and the action selection function are deterministic, i.e., 
 $\sigma_u: M \times \obs \times \act \to M$ and $\sigma_n: M \to \act$.
% %%It was shown in~\cite{CDGH10} that in POMDPs pure strategies are as powerful 
% %%as randomized strategies, hence in sequel we omit discussions about 
% %%pure strategies.
% A strategy is \emph{memoryless (or stationary)} if it is independent of the 
% history but depends only on the current observation, and can be represented
% as a function $\sigma: \obs \to \distr(\act)$.  
% %A policy is \emph{deterministic} if $\sigma_n:M\times\mathcal{O}\rightarrow\Sigma$. 
The general class of strategies is sometimes referred to as the class of \emph{randomized infinite-memory} strategies.
%%policies.
% %We write $\Pi(\mathcal{S})$ for the set of policies on $\mathcal{S}$. We write $\Pi_{FR}(\mathcal{S})$ for the set of randomized policies with finite memories, $\Pi_{D}(\mathcal{S})$ for the set of deterministic policies, and $\Pi_{FD}(\mathcal{S})$ for the set of deterministic policies with finite memories. Given $k\in\mathbb{N}$, $\Pi_{k}(\mathcal{S})$ is the set of policies such that $|M|\leq k$, and $\Pi_{k,D}(\mathcal{S})$ is the set of deterministic policies such that $|M|\leq k$.

\smallskip\noindent{\bf Probability and expectation measures.}
Given a strategy $\sigma$ and a starting state $s$, the unique probability measure obtained given 
$\sigma$ is denoted as $\prb_s^{\sigma}(\cdot)$.
We first define the measure $\mu_s^\sigma(\cdot)$ on cones.
For $w=s$ we have $\mu_s^\sigma(\Cone(w))=1$, and 
for $w=s'$ where $s\neq s'$ we have  $\mu_s^\sigma(\Cone(w))=0$; and 
for $w' = w \cdot a\cdot s$ 
we have 
$\mu_s^\sigma(\Cone(w'))= \mu_s^\sigma(\Cone(w)) \cdot \sigma(w)(a) \cdot \trans(\last(w),a)(s)$. 
By Carath\'eodory's extension theorem, the function $\mu_s^\sigma(\cdot)$
can be uniquely extended to a probability measure $\prb_s^{\sigma}(\cdot)$
over Borel sets of infinite plays~\cite{Billingsley}. 
We denote by $\expect_s^{\sigma}[\cdot]$ the expectation measure associated with the strategy $\straa$.
For an initial distribution $\initd_0$ we have 
$\prb_{\initd_0}^\straa(\cdot) = \sum_{s\in S} \initd_0(s) \cdot \prb_{s}^\straa(\cdot)$
and $\expect_{\initd_0}^\straa[\cdot] = 
\sum_{s\in S} \initd_0(s) \cdot \expect_{s}^\straa[\cdot]$.

\smallskip\noindent{\bf Objectives.}
We consider reachability and total-cost objectives.
\begin{itemize}
\item \emph{Reachability objectives.} 
A \emph{reachability objective} in a POMDP $G$ is a measurable set $\varphi \subseteq \Omega$ 
of plays and is defined as follows: 
given a set $\target \subseteq S$ of \emph{target} states, the \emph{reachability} objective 
$\Reach(\target) = \{ (s_0, a_0, s_1, a_1, s_2 \ldots) \in \Omega \mid 
\exists i \geq 0:  s_i \in \target\}$
requires that a target state in $\target$ is visited at least once. 

\item \emph{Total-cost and finite-length total-cost objectives.}
A \emph{total-cost} objective is defined as follows: 
Let $G$ be a POMDP with a set of absorbing target states $\target$ and a 
\emph{cost} function $\cost : \states \times \act \rightarrow \integer$ 
that assigns integer-valued weights to all states and actions 
such that for all states $t \in \target$ and all actions $a \in \act$ we have $\cost(t,a) =0$.
The total-cost of a play $\rho = (s_0,a_0,s_1,a_1,s_2,a_2,\ldots)$ is 
$\sumcost(\rho)  = \sum_{i=0}^{\infty}\cost(s_i,a_i)$ the sum of the costs of the play.
To analyze total-cost objectives we will also require finite-length total-cost objectives,
that for a given length $k$ sum the total costs upto length $k$; i.e.,
$\sumcost_k(\rho) =  \sum_{i=0}^{k}\cost(s_i,a_i)$.

\end{itemize}

% with a cost function $\cost : \states \times \act \rightarrow \integer$ assigning every pair of a state and 
% an action an integer cost the quantitative objective \emph{sum} is defined as follows: given a play $\rho = (s_0,a_0,s_1,a_1,s_2,a_2,\ldots)$
% the value of the play is the sum $\sumcost(\rho)  = \sum_{i=0}^{\infty}(\cost(s_i,a_i))$.

% In this work we introduce a new \emph{quantitative} objective \emph{cost-effective reachability} $\cereach$, which is a combination
% of the previously defined reachability objective and the quantitative sum objective. Given a set of target states $\target$
% and a cost function $\cost : \states \times \act \rightarrow \integer$,
%  % such that for all statet $t \in \target$ and all actions $a \in \Act$ we have $\cost(t,a) =0$  
%  the quantitative cost-effective reachability objective $\cereach^{\target}_{\cost}: \plays \rightarrow \reals$ is 
% defined as follows:
% Given a play $\rho =(s_0,a_0,s_1,a_1,s_2,a_2,\ldots)$ we have:

% $$ \cereach^{\target}_{\cost}(\rho) = 
% \begin{cases} 
% \infty & \text{if for all }i \in \nat: s_i \not \in \target 
% \end{cases} $$

\smallskip\noindent{\bf Almost-sure winning.}
Given a POMDP $\game$ with a reachability objective $\reach(\target)$ a strategy~$\straa$
is \emph{almost-sure winning} iff $\prb^{\straa}_{\initd_0}(\reach(\target)) = 1$. 
We will denote by $\almost_{\game}(\target)$ the set of almost-sure winning strategies in 
POMDP $\game$ for the objective $\reach(\target)$.
Given a set $U$ such that all states in $U$ have the same observation, a strategy is 
almost-sure winning from $U$, if given the uniform probability distribution $\initd_U$ 
over $U$ we have $\prb_{\initd_U}^\straa(\Reach(T))=1$; i.e., the strategy ensures almost-sure 
winning if the starting belief-support is $U$.
%(i.e., the starting state can be any state in $U$ and the starting state is unknown to the strategy).

\smallskip\noindent{\bf Optimal cost under almost-sure winning and 
approximations.}
Given a POMDP $\game$ with a reachability objective $\reach(\target)$ and a cost function $\cost$ 
we are interested in minimizing the expected total cost before reaching the target set $\target$,
while ensuring that the target set is reached almost-surely. 
Formally, the value of an almost-sure winning strategy $\straa \in \almost_{\game}(\target)$ is 
the expectation $\val(\straa) = \expect_{\initd_0}^{\straa}[\sumcost]$. 
The \emph{optimal cost} is defined as the infimum of expected costs among all almost-sure winning 
strategies: $\optCost = \inf_{\straa \in \almost_{\game}(\target)} \val(\straa)$.
We consider the computational problems of approximating $\optCost$ and 
compute strategies $\straa \in \almost_{\game}(\target)$ such that the 
value $\val(\straa)$ approximates the optimal cost $\optCost$.
Formally, given $\epsilon>0$, the \emph{additive} approximation problem asks 
to compute a strategy $\straa \in  \almost_{\game}(\target)$ such that 
$\val(\straa) \leq \optCost + \epsilon$; and the \emph{multiplicative} 
approximation asks 
to compute a strategy $\straa \in  \almost_{\game}(\target)$ such that 
$\val(\straa) \leq \optCost\cdot(1 + \epsilon)$.

\begin{remark}\label{rem:notation}
We remark about some of our notations.
\begin{enumerate}
\item \emph{Rational costs:} We consider integer costs as compared rational costs,
and given a POMDP with rational costs one can obtain a POMDP with integer costs
by multiplying the costs with the least common multiple of the denominators.
The transformation is polynomial given binary representation of numbers.

\item \emph{Probabilistic observations:} 
Given a POMDP $\game=(S,\act,\trans,\obs,\obsmap,\initd_0)$, the most general type
of the observation function $\obsmap$ considered in the literature is of type 
$S \times \act \rightarrow \distr(\obs)$, i.e., the state and the action 
gives a probability distribution over the set of observations $\obs$. 
We show how to transform the POMDP $\game$ into one where the observation function 
is deterministic and defined on states, i.e., of type $S \rightarrow \obs$ as in our definitions. 
We construct an equivalent POMDP $\game'=(S',\act,\trans',\obs,\obsmap',\initd_0')$
as follows: 
(i)~the new state space is $S' = S \times \obs$; 
(ii)~the transition function $\trans'$ given a state $(s,z) \in S'$ and an action $a$ 
is as follows $\trans'((s,z),a)(s',z') = \trans(s,a)(s') \cdot \obsmap(s',a)(z')$; and
(iii) the deterministic observation function for a state $(s,z) \in S'$ is defined as $\obsmap'((s,z)) 
= z$.
Informally, the probabilistic aspect of the observation function is captured in the 
transition function, and by enlarging the state space with the product with the observations,
we obtain an observation function only on states.
Thus we consider observation on states which greatly simplifies the notation.
 
\item \emph{Strategies:} 
Note that in our definition of strategies, the strategies operate on state action sequences,
rather than observation action sequences. 
However, since we restrict strategies to be observation based, in effect they operate on 
observation action sequences. 

\end{enumerate}
\end{remark}

\section{Approximating $\optCost$ for Integer Costs}
In this section we will show that the problem of approximating the optimal cost $\optCost$ 
is undecidable.
We will show that deciding whether the optimal cost $\optCost$ is $- \infty$ or not
is undecidable in POMDPs with integer costs.
We present a reduction from the standard undecidable problem for probabilistic 
finite automata (PFA). 
A PFA $\automaton = (\states, \act, \trans,F,s_0)$ is a special case of a POMDP 
$\game = (\states,\act,\trans,\obs,\obsmap,s_0)$ with a single observation 
$\obs=\set{z}$ such that for all states $s\in \states$ we have $\obsmap(s)=z$.
Moreover, the PFA proceeds for only finitely many steps, and has a set $F$ of 
desired final states.
The \emph{strict emptiness problem} asks for the existence of a strategy $w$ 
(a finite word over the alphabet $\act$) such that the measure of the runs 
ending in the desired final states $F$ is strictly greater than $\frac{1}{2}$;
and the strict emptiness problem for PFA is undecidable~\cite{PazBook}.

%\begin{problem}
%\emph{\textbf{(The strict emptiness problem).}} 
%Given a probabilistic finite automaton $\automaton$ decide whether there exists a word $w$ 
%such that $\prb_{\automaton}(w)> \frac{1}{2}$.
%\end{problem}

\smallskip\noindent{\bf Reduction.}
Given a PFA $\automaton = (\states, \act, \trans,F,s_0)$ 
%with a total transition function 
we construct a POMDP $\game = (\states',\act',\trans', \obs, \obsmap,s'_0)$ 
with a cost function $\cost$ and a target set $\target$ such that there exists a word 
$w \in \act^*$  accepted with probability strictly greater than $\frac{1}{2}$ 
in PFA $\automaton$ iff the optimal cost
in the POMDP $\game$ is $-\infty$.
Intuitively, the construction of the POMDP $\game$ is as follows: for every 
state $s \in \states$ of $\automaton$ we construct a pair of states $(s,1)$ and 
$(s,-1)$ in $\states'$ with the property 
that $(s,-1)$ can only be reached with a new action~$\$$ (not in $\act$) 
played in state $(s,1)$. The transition function $\trans'$ from the state 
$(s,-1)$ mimics the transition function $\trans$, i.e., 
$\trans'((s,-1),a)((s',1)) = \trans(s,a)(s')$.
The cost $\cost$ of $(s,1)$ (resp. $(s,-1)$) is $1$ (resp. $-1$), ensuring 
the sum of the pair to be $0$. 
We add a new available action $\#$ that when played in a final state reaches a 
newly added state $\good \in \states'$, 
and when played in a non-final state reaches a newly added state $\bad \in \states'$.
For states $\good$ and $\bad$ given action $\#$ the next state is the 
initial state; with negative cost $-1$ for $\good$ and positive cost $1$ for $\bad$. 
We introduce a single absorbing target state $\target = \{\targetstate\}$ and 
give full power to the player to decide when to reach the target state from the 
initial state, i.e., we introduce a new action $\surd$ that when played in the initial state 
deterministically reaches the target state $\targetstate$.

An illustration of the construction on an example is depicted on Figure~\ref{fig:incompatible}.
Whenever an action is played in a state where it is not available, the POMDP reaches a losing absorbing state, 
i.e., an absorbing state with cost $1$ on all actions, and for brevity we omit transitions
to the losing absorbing state.
The formal construction of the POMDP $\game$ is as follows:
\begin{itemize}
\item $\states' = (\states \times \{-1,1\}) \cup \{\good,\bad, \targetstate\}$,

\item $s'_0 = (s_0,1)$,

\item $\act' = \act \cup \{\#,\$, \surd\}$,

\item The actions $a\in \act\cup \{\#\}$ in states $(s,1)$ (for $s \in S$) 
lead to the losing absorbing state; the action $\$$ in states $(s,-1)$ (for 
$s \in S$) leads to the losing absorbing state; and the actions 
$a\in \act\cup \{\$\}$ in states $\good$ and $\bad$ lead 
to the losing absorbing state. The action $\surd$ played in any state other than the
initial state $s'_0$ also leads to the losing absorbing state.
The other transitions are as follows:
For all $s \in \states$: 
(i)~$\trans'((s,1),\$)((s,-1)) = 1$, 
(ii)~for all $a \in \act$ we have $\trans'((s,-1),a)((s',1)) = \trans(s,a)(s')$,
and 
(iii)~for action $\#$ and $\surd$ we have 
	 $$
	\trans'((s,-1),\#)(\good) = 
	\begin{cases}
	1 & \text{ if $s \in F$;} \\
	0 & \text{ otherwise;}
	\end{cases} \hspace{4em} 
	\trans'((s,-1),\#)(\bad) = 
	\begin{cases}
	1 & \text{ if $s \not \in F$;} \\
	0 & \text{ otherwise;}
	\end{cases}$$
	$$\trans'(\good,\#)(s'_0) = 1; \hspace{4em} \trans'(\bad,\#)(s'_0) = 1; \hspace{4em} \trans'(s'_0,\surd)(\targetstate)=1;$$
	\item there is a single observation $\obs = \{o\}$, and all the states $s \in \states'$ have $\obsmap(s)=o$.
	\end{itemize}
We define the cost function $\cost$ assigning only two different costs and only as a function of the 
state, i.e., $\cost:\states' \rightarrow \{-1,1\}$ and show the 
undecidability even for this special case of cost functions. 
For all $s \in \states$ the cost is $ \cost((s,-1)) = -1$, and similarly 
$\cost((s,1))=1$, and the remaining states have costs 
$\cost(\good) =-1$ and $\cost(\bad) = 1$.
The absorbing target state has cost~0; i.e., $\cost(\targetstate)=0$. 
Note that though the costs are assigned as function of states,
the costs appear on the out-going transitions of the respective states.
% We now establish the correctness of the reduction.

\smallskip\noindent{\bf Intuitive proof idea.} 
The basic idea of the proof is as follows: 
Consider a word $w$ accepted by the PFA with probability at least $\nu > \frac{1}{2}$.
Let the length of the word be $|w| = n$, and $w[i]$ denote the $i^{th}$ letter in $w$. 
Consider a strategy in the POMDP $ u = ( \$ \: w[1] \: \$ \: w[2] \ldots \: \$ w[n]  \: \# \: \#)^{k}\surd$ 
for some constant $k \geq 0$; that plays alternately the letters in $w$ and $\$$, then two $\#$'s, 
repeat the above $k$ times, and finally plays $\surd$. 
For any $\tau<0$, for $k \geq \frac{\tau-1}{1-2\cdot \nu}$, the expected total cost is below $\tau$.
Hence if the answer to the strict emptiness problem is yes, then the optimal cost is $-\infty$.
Conversely, if there is no word accepted with probability strictly greater than $\frac{1}{2}$, 
then the expected total cost between consecutive visits to the starting state is positive, 
and hence the optimal cost is at least~1.
We now formalize the intuitive proof idea.

 \begin{figure}[ht]
\begin{center}
\resizebox{12cm}{!}{
\begin{tikzpicture}[>=latex]
\tikzstyle{every node}=[font=\small]

\node[Player1,initial,initial text=]  (s0) {$s_0$};
\node[Final, inner sep=0.25cm,right of=s0,xshift=40]  (a) {$s$};
\draw[->]{
(s0) edge[] node[auto] {a} (a)
(a) edge[loop] node[above] {b} (a)
};

\node[Player1,initial,initial text=,right of=a, xshift=120] (s01) {$s_0,1$};
\node[Player1,right of=s01,xshift=50] (s00) {$s_0,-1$};
\node[Player1,right of=s00,xshift=50] (s11) {$s,1$};
\node[Player1,above of=s11,yshift=25] (s10) {$s,-1$};
\node[Player1,inner sep=0.01cm,above of=s00,yshift=25] (good) {$\good$};
\node[Player1,inner sep=0.1cm,left of=s01,xshift=-50] (bad) {$\bad$};
\node[Final,inner sep=0.05cm,above of=s01,yshift=25] (target) {$\targetstate$};

\draw[->]{
(s01) edge[] node[auto] {$\$$} (s00)
(s00) edge[] node[auto] {$a$} (s11)
(s11) edge[bend left] node[auto] {$\$$} (s10)
(s10) edge[bend left] node[auto] {b} (s11)
(s10) edge[] node[auto] {$\#$} (good)
(good) edge[] node[left] {$\#$} (s01)
(s00) edge[bend left] node[auto] {$\#$} (bad)
(bad) edge[bend left] node[auto] {$\#$} (s01)
(s01) edge[] node[left] {$\surd$} (target)
};
\end{tikzpicture}
}
\caption{PFA $\automaton$ to a POMDP $\game$}
\label{fig:incompatible}
\end{center}
% \vspace{-1.5em}
\end{figure}

\begin{lemma}\label{lem:finquant_1}
If there exists a word $w\in\act^*$ that is accepted with probability strictly greater 
than $\frac{1}{2}$ in $\automaton$, then the optimal cost $\optCost$
in the POMDP $\game$ is $-\infty$.
\end{lemma}
\begin{proof}
Let $w \in \act^*$ be a word that is accepted in $\automaton$ with probability  
$\nu > \frac{1}{2}$ and let $\tau \in \reals$ be any negative real-number threshold. 
We will construct a strategy in POMDP $\game$ ensuring that the target state $\targetstate$ 
is reached almost-surely and the value of the strategy is below $\tau$. 
As this is will be true for every $\tau < 0 $ it will follow that the optimal cost 
$\optCost$ is $-\infty$. 

Let the length of the word be $|w| = n$. 
We construct a pure finite-memory strategy in the POMDP $\game$ as follows: 
We denote by $w[i]$ the $i^{th}$ action in the word $w$. The finite-memory strategy we construct is specified as a 
word  $ u = ( \$ \: w[1] \: \$ \: w[2] \ldots \: \$ w[n]  \: \# \: \#)^{k}\surd$ for some constant $k \geq 0$,
i.e., the strategy plays alternately the letters in $w$ and $\$$, then two $\#$'s, 
repeat the above $k$ times, and finally plays $\surd$. 
Observe that by the construction of the POMDP $\game$, the sequence of costs (that appear on 
the transitions) is $(1,-1)^n$ followed by (i)~$-1$ with probability $\nu$ (when 
$F$ is reached), and (ii)~$+1$ otherwise; and the whole sequence is repeated 
$k$ times.
%Also observe that once the pure finite-memory strategy is fixed we obtain a 
%Markov chain with a single recurrent class since the starting state belongs to 
%the recurrent class and all states reachable from the starting state 
%form the recurrent class.

Let $r_1, r_2, r_3, \ldots r_m$ be the finite sequence of costs and 
$s_j  =  \sum_{i=1}^{j}r_i$. 
The sequence of costs can be partitioned into blocks of length 
$2\cdot n+1$, intuitively corresponding to the transitions of a single run on the word 
$(\$ \: w[1] \: \$ \: w[2] \ldots \: \$ w[n] \: \# \: \#)$. 
We define a random variable $X_i$ denoting the sum of costs of the $i^{th}$ 
block in the sequence, i.e., with probability $\nu$ for all $i$ the value of 
$X_i$ is $-1$ and with probability $1-\nu$ the value is 
$1$. 
The expected value of $X_i$ is therefore equal to $\Exp[X_i]=1-2\cdot \nu $, 
and as we have that $\nu > \frac{1}{2}$ it follows that $\Exp[X_i] < 0$.
The fact that %we have a single recurrent class and 
after the $\#\#$ the initial state is reached implies that the random variable 
sequence $(X_i)_{0 \leq i\leq k}$ is a finite sequence of i.i.d's. 
By linearity of expectation we have that the expected total cost of the word $u$ is $k \cdot \Exp[X_i]$ 
plus an additional $1$ for the last $\surd$ action. 
Therefore, by choosing an appropriately large $k$ (in particular for $k \geq \frac{\tau-1}{1-2\cdot\nu}$) 
we have the expected total cost is below $\tau$.
As playing the $\surd$ action from the initial state reaches the target state $\targetstate$ almost-surely, 
and after $\#\#$ the initial state is reached almost-surely, we have that by playing according to the strategy 
$u$ the target state $\targetstate$ is reached almost-surely.
The desired result follows.
\hfill\qed
%%As the expectation can be lowered below any given bound $\tau \in \reals$ the result follows.
\end{proof}

We now show that pure finite-memory strategies are sufficient for the POMDP we 
constructed from the probabilistic automata, and then prove a lemma that 
proves the converse of Lemma~\ref{lem:finquant_1}.

\begin{lemma}\label{lemm:pure-fin-mem}
Given the POMDP $\game$ of our reduction from the PFA, if there is a randomized 
(possibly infinite-memory) strategy $\straa$ with $\val(\straa)<1$,
then there exists a pure finite-memory strategy $\straa'$ with $\val(\straa')<1$.
\end{lemma}
\begin{proof}
Let $\straa$ be a randomized (possibly infinite-memory) strategy with the expected total cost strictly less than $1$.
As there is a single observation in the POMDP $\game$ constructed in our reduction, the strategy does not receive 
any useful feedback from the play, i.e., the memory update function $\straa_u$ always receives as one of the 
parameters the unique observation.
% Note that the strategy needs to play the pair $\#\#$ of actions infinitely often with probability $1$, i.e., 
Note that with probability $1$ the resolving of the probabilities in the strategy $\straa$ leads to finite words 
of the form $\rho  = w_1 \#\# w_2 \#\# \ldots \#\# w_n \#\# \surd$,
as otherwise the target state $\targetstate$ is not reached with probability $1$.
From each such word $\rho$ we extract the finite words $w_1,w_2, \ldots, w_n$ that occur in $\rho$,
and consider the union of all such words as $W(\rho)$, and then consider the union $W$ over all such words
$\rho$.
%%and consider the set of all finite words $W \subseteq \act'^*$ that are played by the strategy with positive in an infinite word, note that this set might be infinite.
We consider two cases:
\begin{enumerate}
\item If there exists a word $v$ in $W$ such that the expected total cost after playing the word $v\#\#$ is strictly less than $0$, 
then the pure finite-memory strategy $v \#\#\surd$ ensures that the expected total cost strictly less than $1$.
        
\item Assume towards contradiction that for all the words $v$ in $W$ the expected total cost of 
$v\#\#$ is at least $0$.
Then with probability $1$ resolving the probabilities in the strategy $\straa$ leads to finite words of the form 
$\wb{w} = w_1 \#\# w_2 \#\# \ldots w_n \#\# \surd$, 
where each word $w_i$ belongs to  $W$, that is played on the POMDP $\game$.
Let us define a random variable $X_i$ denoting the sum between $i$ and $(i+1)$-th occurrence of $\#\#$.
The expected total cost of $w_i\#\#$ is $\Exp[X_i]$ and is at least $0$ for all $i$.
Therefore the expected cost of the sequence $\wb{w}$ (which has $\surd$ in the end with cost~1) 
is at least $1$.
% 
% $$ \Exp(\liminf_{n\rightarrow \infty} \frac{1}{n} \sum\limits_{i=0}^{n} X_i).$$
% 
% Since $X_i$'s are non-negative measurable functions, by Fatou's lemma~\cite[Theorem~3.5, page 16]{Durrett} 
% that shows the integral of limit inferior of a sequence of non-negative 
% measurable functions is at most the limit inferior of the integrals of these 
% functions, we have the following inequality:
% 
% $$ \Exp(\liminf_{n\rightarrow \infty} \frac{1}{n} \sum\limits_{i=0}^{n} X_i) \leq  \liminf_{n \rightarrow \infty} \Exp(\frac{1}{n} \sum\limits_{i=0}^{n} X_i) \leq \frac{1}{2}.$$
% 
% Note that since the strategy $\straa$ is having the expected cost strictly smaller than $1$, 
% then the expected value of costs must be strictly greater than $\frac{1}{2}$. 
Thus we arrive at a contradiction.
Hence, there must exist a word $v$ in $W$ such that $v\#\#$ has an expected total cost
strictly less than $0$.
\end{enumerate}
This concludes the proof.
\hfill\qed
\end{proof}

\begin{lemma}
If there exists no word $w\in\act^*$ that is accepted with probability strictly greater 
than $\frac{1}{2}$ in $\automaton$, then the optimal cost $\optCost=1$.

\end{lemma}
\begin{proof}
We will show that playing $\surd$ is an optimal strategy. It reaches the target state $\targetstate$
almost-surely with an expected total cost of $1$. 
Assume towards contradiction that there exists a strategy (and by Lemma~\ref{lemm:pure-fin-mem}, 
a pure finite-memory strategy) $\straa$ with the expected total cost strictly less than $1$.
Observe that as there is only a single observation in the POMDP $\game$ the pure finite-memory 
strategy $\straa$ can be viewed as a finite word of the form $ w_1 \#\# w_2 \#\# \ldots w_n \#\# \surd$.
%  where 
% $u,v$ are finite words from $\act'$. 
% Note that $v$ must contain the subsequence $\# \#$, as otherwise the $\limavg$ 
% would be only $\frac{1}{2}$. 
% Similarly, before every letter $a \in \act$ in the words $u,v$, the strategy 
% must necessarily play the $\$$ action, as otherwise the losing absorbing 
% state is reached.
% In the first step we align the $\#\#$ symbols in $v$. Let us partition the 
% word $v$ into two parts $v = y \cdot x$ such that $y$ is the shortest prefix 
% ending with $\#\#$. Then the ultimately periodic word $u \cdot y \cdot 
% (x\cdot y)^{\omega} = u \cdot v^\omega$ is also a strategy ensuring 
% almost-surely $\limavghalf$.
% Due to the previous step we consider $u'=u \cdot y$ and $v'= x\cdot y$, and thus have that 
% $v'$ is of the form:
% 	 $$ \$  w_1[1]  \$  w_1[2] \ldots \$ w_1[n_1]  \#\# \$ w_2[1] \$ w_2[2] \ldots \$ w_2[n_2] \# \# \ldots \$ w_m[1]  \$  w_m[2] \ldots \$ w_m[n_m] \# \#$$
 We extract the set of words $W =  \{w_1,w_2, \ldots, w_n\}$ from the strategy $\straa$. By the condition of the lemma,
 there exists no word accepted in the PFA $\automaton$ with probability strictly greater that $\frac{1}{2}$.
 %, i.e., we have that $\forall w \in W: \prb_{\automaton}(w) \leq \frac{1}{2}$.
% Assume towards contradiction that all the words in the set $W$ are accepted in the PFA $\automaton$ with probability at most $\frac{1}{2}$. 
%i.e., for all $w \in W$ we have $\prb_{\automaton}(w) \leq \frac{1}{2}$. 
As in Lemma~\ref{lem:finquant_1} we define a random variable $X_i$ denoting the sum of costs after reading $w_i\#\#$. 
It follows that the expected value of $\Exp[X_i] \geq 0$ for all $0 \leq i \leq n$. By using the linearity of expectation
we have the expected total cost of  $w_1 \#\# w_2 \#\# \ldots w_n \#\#$ is at least $0$, and 
hence the expected total cost of the strategy is at least~1 due to the cost of the last $\surd$ action.
Thus we have a contradiction to the assumption that the expected total cost of strategy $\straa$ is strictly 
less than $1$.
\hfill\qed
\end{proof}

% of $u\cdot v^{\omega}$ is $\Exp(X_i)$, and hence it 
% is not possible as $u\cdot v^{\omega}$  is an almost-sure winning strategy for 
% the objective $\limavghalf$.
% We reach a contradiction to the assumption that all the words in $W$ are accepted with probability at most $\frac{1}{2}$ in $\automaton$. 
% Therefore there exists a word $w \in W$ that is accepted in ${\automaton}$ 
% with probability strictly greater than $\frac{1}{2}$, which concludes the proof.

The above lemmas establish that if the answer to the strict emptiness problem for
PFA is yes, then the optimal cost in the POMDP is $-\infty$; and otherwise the 
optimal cost is~1.
Hence in POMDPs with integer costs determining whether 
the optimal cost $\optCost$ is $-\infty$ or~1 is undecidable, and 
thus the problem of approximation is also undecidable.

\begin{theorem}\label{thm:undec}
The problem of approximating the optimal cost $\optCost$ in POMDPs with integer costs is 
undecidable for all $\epsilon>0$ both for additive and multiplicative approximation.
\end{theorem}

%\begin{remark}
%%Traditionally the next step would be to relax the definition of the cost function \cite{DDGRT10}.
%Note that in our reduction the costs of the transitions can be inferred from the action itself
% The definition
%of \emph{visible costs} is a function $\obs \times \act \rightarrow \integer$ rather then $\states \times \act \rightarrow \integer$ we use in our work. 
%Note that in the POMDP of our construction the undecidability proof still works even if we allow the Player to observe the costs. It follows that the problem 
%remains undecidable even in this standard simplified setting.
%\end{remark}

\newcommand{\unif}{\mathsf{unif}}
\newcommand{\U}{\wh{U}}

\newcommand{\famPOMDP}{\mathcal{F}}
\newcommand{\subPOMDP}{\mathcal{L}}

\section{Approximating $\optCost$ for Positive Costs}
In this section we consider POMDPs with positive cost functions, 
i.e., $\cost: \states \times \act \rightarrow \nat$ instead of 
$\cost: \states \times \act \rightarrow \integer$.
Note that the transitions from the absorbing target states have cost~0
as the goal is to minimize the cost till the target set is reached, 
and also note that all other transitions have cost at least~1. 
We established (in Theorem~\ref{thm:undec}) that for integer costs the problem of 
approximating the optimal cost is undecidable, and 
in this section we show that for positive cost functions the approximation 
problem is decidable.
We first start with a lower bound on the optimal cost.

\subsection{Lower Bound on $\optCost$}
We present a double-exponential lower bound on $\optCost$ with respect to the number of states of the POMDP. 
We define a family of POMDPs $\famPOMDP(n)$, for every $n$, with a single target state, 
such that there exists an almost-sure winning strategy, 
and for every almost-sure winning strategy the expected number of steps to reach the target state is double-exponential 
in the number of states of the POMDP.
Thus assigning cost~1 to every transition we obtain the double-exponential lower bound.

\smallskip\noindent\textbf{Preliminary.}
The action set we consider consists of two symbols $\act = \{a,\#\}$. 
The state space consists of an initial state $s_0$, a target state $\targetstate$, 
a losing absorbing state $\bad$ and a set of $n$ sub-POMDPs $\subPOMDP_i$ for $1 \leq i \leq n$. 
Every sub-POMDP $\subPOMDP_i$ consists of states $Q_i$ that form a loop of $p(i)$ states $q^i_1, q^i_2, \ldots q^i_{p(i)}$, 
where $p(i)$ denotes the $i$-th prime number and $q^i_1$ is the initial state of the sub-POMDP. 
For every state $q^i_j$ (for $ 1 \leq j \leq p(i)$) the transition function under action $a$ 
moves the POMDP to the state $q^i_{(j \mod p(i)) + 1}$ with probability $\frac{1}{2}$ and 
to the initial state $s_0$ with the remaining probability $\frac{1}{2}$. 
The action $\#$ played in the state $q^i_{p(i)}$ moves the POMDP to the target state $\targetstate$ 
with probability $\frac{1}{2}$ and to the initial state $s_0$ with the remaining probability $\frac{1}{2}$. 
For every other state in the loop $q^i_j$ such that $1 \leq j < p(i)$ the POMDP moves under action $\#$ 
to the losing absorbing state $\bad$ with probability $1$. 
The losing state $\bad$ and the target state $\targetstate$ are absorbing and have a self-loop under 
both actions with probability $1$.

\smallskip\noindent\textbf{POMDP family $\famPOMDP(n)$.}
Given an $n \in \mathbb{N}$ we define the POMDP $\famPOMDP(n)$ as follows: 
\begin{itemize}
\item The state space $S = Q_1 \cup Q_2 \cup \ldots Q_n  \cup \{s_0,\bad,\targetstate\}$.
\item There are two available actions $\act = \{a,\#\}$.
\item The transition function is defined as follows: action $a$ in the initial state leads to $\bad$ with probability $1$ 
and action $\#$ in the initial state leads with probability $\frac{1}{n}$ to the initial state of the sub-POMDP $\subPOMDP_i$ 
for every $1 \leq i \leq n$. 
The transitions for the states in the sub-POMDPs are described in the previous paragraph.
\item All the states in the sub-POMDPs $\subPOMDP_i$ do have the same observation $z$. 
The remaining states $s_0$, $\bad$, and $\targetstate$ are visible, i.e., each of these three states 
has its own observation.
\item The initial state is $s_0$.
\end{itemize}
The cost function $\cost$ is defined as follows: 
the self-loop transitions at $\targetstate$ have cost~0 and all other transitions have cost~1.
An example of the construction for $n=2$ is depicted in Figure~\ref{fig:p2}, where we omit the 
losing absorbing state $\bad$ and the transitions leading to $\bad$ for simplicity.

\smallskip\noindent{\bf Intuitive proof idea.} 
For a given $n \in \mathbb{N}$ let $p^*_n = \prod_{i=1}^{n} p(i)$ and 
$p^+_n=\sum_{i=1}^n p(i)$ denote the product and the sum of the first $n$ prime numbers,
respectively. 
Note that $p^*_n$ is exponential is $p^+_n$.
An almost-sure winning strategy must play as follows: in the initial state $s_0$ it plays
$\#$, and then if it observes the observation $z$ for at least $p^*_n$ consecutive steps, 
then for each step it must play action $a$, and at the $p^*_n$ step it can play action $\#$.
Hence the probability to reach the target state in $p^*_n$ steps is at most 
$(\frac{1}{2} \cdot (\frac{1}{2})^{p^*_n})$; and hence the expected number of steps to reach the 
target state is at least $p^*_n \cdot 2 \cdot 2^{p^*_n}$.
The size of the POMDP is polynomial in $p^+_n$ and thus the expected total cost is
double exponential.
% to reach the target state is at least double exponential in the size of the POMDP.

\begin{lemma}
\label{lem:lower_bound}
There exists a family $(\famPOMDP(n))_{n \in \mathbb{N}}$ of POMDPs of size $\mathcal{O}(p(n))$ for a polynomial $p$ with a reachability objective, 
such that the following assertion holds: There exists a polynomial $q$ such that for every almost-sure winning strategy 
the expected total cost to reach the target state is at least $2^{2^{q(n)}}$.
\end{lemma}
\begin{proof}
%For a given $n \in \mathbb{N}$ let $p^*_n = \prod_{i=1}^{n} p(i)$ and 
%$p^+_n=\sum_{i=1}^n p(i)$ denote the product and sum of the first $n$ prime numbers,
%respectively. 
%%Note that $p^*_n$ is exponential is $p^+_n$.
For $n \in \nat$, consider the POMDP $\famPOMDP(n)$, and an almost-sure winning strategy in the POMDP.
In the first step the strategy needs to play the $\#$ action from $s_0$, as otherwise the losing absorbing state is reached. 
The POMDP reaches the initial state of the sub-POMDPs $\subPOMDP_i$, for all $i$, with positive probability. 
As all the states in the sub-POMDPs have the same observation $z$, the strategy cannot base its decision on the current sub-POMDP. 
The strategy has to play the action $a$ until the observation $z$ is observed for $p^*_n$ steps in a row before playing action $\#$. 
If the strategy plays the action $\#$ before observing the sequence of $z$ observations $p^*_n$ times, then it reaches the losing absorbing 
state with positive probability (and would not have been an almost-sure winning strategy). 
This follows from the fact that there is a positive probability of being in a sub-POMDP, where the current state is not the last one of the loop.  
Hence an almost-sure winning strategy must play $a$ as long as the length of the sequence of the observation $z$ 
is less than $p^*_n$ consecutive steps.
Note that in between the POMDP can move to the initial state $s_0$, and the strategy restarts.
In every step of the sub-POMDPs, with probability $\frac{1}{2}$ the initial state is reached,
and the next state is in the sub-POMDP with probability $\frac{1}{2}$.
After observing the $z$ observation for $p^*_n$ consecutive steps, the strategy can play the action $\#$ that moves the POMDP to the target state $\targetstate$ with probability $\frac{1}{2}$ 
and restarts the POMDP with the remaining probability $\frac{1}{2}$. 
Therefore the probability of reaching the target state in $p^*_n$ steps is at most $(\frac{1}{2} \cdot (\frac{1}{2})^{p^*_n})$;
and hence the expected number of steps to reach the target state is at least $p^*_n \cdot 2 \cdot 2^{p^*_n}$.
The size of the POMDP is polynomial in $p^+_n$ and hence it follows that the expected total cost to
reach the target state is at least double exponential in the size of the POMDP. 
\hfill\qed
\end{proof}

\begin{figure}[ht]
\begin{center}
% \resizebox{12cm}{!}{
\begin{tikzpicture}[>=latex,remember picture,
  inner/.style={circle,draw=blue!50,fill=blue!20,thick,inner sep=2pt},
  outer/.style={draw=black,dashed,fill=green!10,thick,inner sep=6pt}
  ]]

   \node[outer] (p1) {
    \begin{tikzpicture}
      \node[State,solid] (p11) {$q^1_1$};
      \node[State, solid,below of=p11,yshift=-35] (p12) {$q^1_2$};
    \end{tikzpicture}
  };
  \node[outer,right=of p1,xshift=60] (p2) {
    \begin{tikzpicture}
      \node[State,solid] (p21) {$q^2_1$};
      \node[State,solid,below of=p21,xshift=30,yshift=-35] (p22) {$q^2_2$};
      \node[State,solid,below of=p21,xshift=-30,yshift=-35] (p23) {$q^2_3$};
     \end{tikzpicture}
  };

  \node[State,solid,below of=p1, yshift=-3,xshift=55] (smil) {$T$};
  \node[State, solid, above of=p1,initial, initial text=, yshift=60,xshift=55] (s) {$s_0$};
  % \node[right of=s,scale=2.5, xshift=22] (frown) {$\frownie$};

\draw[->]{
(s) edge[bend left=20] node[left,yshift=5,xshift=3] {$\#,\frac{1}{2}$} (p11)
(s) edge[bend right=20] node[above,xshift=5] {$\#,\frac{1}{2}$} (p21)
(p11) edge[solid,bend left] node[right,xshift=-20,yshift=-10] {$a, \frac{1}{2}$} (p12)
(p11) edge[solid,bend left] node[left] {$a, \frac{1}{2}$} (s)
(p12) edge[solid,bend left] node[left,xshift=20,yshift=10] {$a, \frac{1}{2}$} (p11)
(p12) edge[solid,bend right=20] node[below,text width=1cm,xshift=8,yshift=-5] {$a, \frac{1}{2}$ $\#, \frac{1}{2}$} (s)
(p12) edge[solid] node[below] {$\#, \frac{1}{2}$} (smil)
(p23) edge[solid] node[below] {$\#, \frac{1}{2}$} (smil)
(p21) edge[solid,bend left] node[right,xshift=-20] {$a, \frac{1}{2}$} (p22)
(p22) edge[solid,bend left] node[above] {$a, \frac{1}{2}$} (p23)
(p23) edge[solid,bend left] node[left,xshift=20] {$a, \frac{1}{2}$} (p21)
(p21) edge[solid,bend right] node[right] {$a, \frac{1}{2}$} (s)
(p22) edge[solid,bend right=60] node[right] {$a, \frac{1}{2}$} (s)
(p23) edge[solid,bend left=15] node[below,text width=1cm,xshift=5,yshift=-5] {$a, \frac{1}{2}$  $\#, \frac{1}{2}$} (s)
(smil) edge[loop above] node[right,text width=1cm,xshift=5,yshift=-5] {$ a,\#$} (smil)
% (p23) edge[solid,bend right=75] node[right,text width=1cm] {aaaaaaa bbbbbb} (s)
  };
  \end{tikzpicture}

\caption{POMDP $\famPOMDP(2)$}
\label{fig:p2}
\end{center}
% \vspace{-1.5em}
\end{figure}

\subsection{Upper Bound on $\optCost$}
In this section we will present a double-exponential upper bound on 
$\optCost$.

\smallskip\noindent{\bf Almost-sure winning belief-supports.}
Let $\beliefset(\game)$ denote the set of all  belief-supports in a POMDP $\game$, i.e.,
$\beliefset(\game) = \{ U \subseteq S \mid \exists z \in \obs : U \subseteq \obsmap^{-1}(z) \}.$
Let $\beliefsetwin(\game, \target)$ denote the set of almost-sure winning 
belief-supports, 
i.e., 
$\beliefsetwin(\game, \target) = \set{ U \in \beliefset(\game) \mid \text{ there exists an almost-sure winning strategy from $U$ }}$,
i.e., there exists an almost-sure winning strategy with initial distribution $\initd_U$ that is the uniform distribution over~$U$.

\smallskip\noindent{\bf Restricting to $\beliefsetwin(\game, \target)$.}
In the sequel without loss of generality we will restrict ourselves to 
belief-supports in $\beliefsetwin(\game, \target)$:
since from belief-supports outside $\beliefsetwin(\game, \target)$ there exists no almost-sure winning strategy,
all almost-sure winning strategies with starting belief-support in 
$\beliefsetwin(\game, \target)$ will ensure
that belief-supports not in $\beliefsetwin(\game, \target)$ are never reached.

\smallskip\noindent{\bf Belief updates.}
Given a belief-support $U \in \beliefset(\game)$, an action $a \in \act$, and an observation $z \in \obs$ 
we denote by $\update(U,z,a)$ the updated belief-support. Formally, the set $\update(U,z,a)$ is defined as follows:
$\update(U,z,a) = \bigcup_{s' \in U} \supp(\trans(s',a)) \cap \obsmap^{-1}(z)$.
The set of belief-supports reachable from $U$ by playing an action $a \in \act$ is denoted by $\update(U,a)$.
Formally, $\update(U,a) = \{U' \subseteq \states \mid  \exists z \in \obs: U' =  \update(U,z,a)) \wedge U' \neq \emptyset\}$.

\smallskip\noindent{\bf Allowed actions.}
Given a POMDP $\game$ and a belief-support $U \in \beliefsetwin(\game,\target)$, 
we consider the set of actions that are guaranteed to keep the next 
belief-support $U'$ in 
$\beliefsetwin(\game,\target)$ and refer these actions as \emph{allowed or safe}. 
The framework that restricts playable actions was also considered in~\cite{CK13}.
Formally we consider the set of allowed actions as follows:
Given a belief-support $U \in \beliefsetwin(\game,\target)$ we define 
%%by $\allow(U)$ the set of actions that when played in belief $U$ ensure that the next belief is in $\beliefsetwin(\game,\target)$, i.e.,
$\allow(U) = \{ a \in \act \mid \forall U' \in \update(U,a): U' \in \beliefsetwin(\game,\target) \}$.
%%\end{itemize}

\begin{comment}
\begin{lemma}
\label{lem:allownonempty}
For all $U \in \beliefsetwin(\game,\target)$ we have $\allow(U) \neq \emptyset$.
\end{lemma}
\begin{proof}
Assume towards contradiction that there exists $U \in \beliefsetwin(\game,\target)$ such that $\allow(U) = \emptyset$. 
It follows from the results of~\cite{CDHR06} that belief-based almost-sure winning strategies exist in POMDPs.
Let $\straa$ be an belief-based almost-sure winning strategy and $a$ be an action played with positive probability 
by $\straa$ in belief $U$. 
As $a \not \in \allow(U)$, there exists a belief $U' \in \update(U,a)$ such that $U' \not \in \beliefsetwin(\game,\target)$. 
As $U'$ is in $\update(U,a)$ it follows that there exists an observation $o \in \obs$ such that $U' \in \update(U,o,a)$. 
As $U'$ is not empty it follows that there is a positive probability
of observing observation $o$ when playing action $a$ in belief $U$ leading to a new belief 
$U' \not \in \beliefsetwin(\game,\target)$. 
It follows that under action $a$, given the current belief is $U$, the next belief is $U'$ with 
positive probability.
By definition, for all beliefs $U'$ that does not belong to $\beliefsetwin(\game,\target)$, 
if the starting belief is $U'$, then for all strategies the probability to reach $\target$
is strictly less than~1.
Hence if $U'$ is reached with positive probability from $U$ under action $a$, we reach
a contradiction that $\straa$ is almost-sure winning for the starting belief $U$.
The desired result follows.
\hfill\qed
\end{proof}
\end{comment}

We now show that almost-sure winning strategies must only play allowed actions.
An easy consequence of the lemma is that for all belief-supports $U$ in 
$\beliefsetwin(\game,\target)$, there is always an allowed action.

\begin{lemma}
\label{lem:play_allow}
Given a POMDP with a reachability objective $\reach(\target)$, consider a strategy 
$\straa$ and a starting belief-support in $\beliefsetwin(\game, \target)$.
Given $\straa$, if for a reachable belief-support $U \in \beliefsetwin(\game, \target)$ the strategy 
$\straa$ plays an action $a$ not in $\allow(U)$ with positive probability,
then $\straa$ is not almost-sure winning for the reachability objective.
\end{lemma}
\begin{proof}
Assume the strategy $\straa$ reaches the belief-support $U$ and plays an action $a \not \in  \allow(U)$.
Since the belief-support $U$ is reachable, it follows that given the strategy $\straa$ when the belief-support is 
$U$, all states in $U$ are reached with positive probability, i.e., given the strategy the belief-support $U$ is reached 
with positive probability. 
It follows from the definition of $\allow$ that there exists a belief-support $U' \in \update(U,a)$ that is not in $\beliefsetwin(\game,\target)$. 
By the definition of $\update$ there exists an observation $z \in \obs$ such that $U' = \update(U,z,a)$ and $U' \not = \emptyset$. 
It follows that by playing $a$ in belief-support $U$, there is a positive probability of observing observation $z$ and reaching 
belief-support $U'$ 
that does not belong to $\beliefsetwin(\game,\target)$. 
It follows that under action $a$, given the current belief-support is $U$, the next belief-support is $U'$ with 
positive probability.
By definition, for all belief-supports $U'$ that does not belong to $\beliefsetwin(\game,\target)$, 
if the starting belief-support is $U'$, then for all strategies the probability to reach $\target$
is strictly less than~1.
Hence if $U'$ is reached with positive probability from $U$ under action $a$, 
then $\straa$ is not almost-sure winning. 
The desired result follows.
\hfill\qed
\end{proof}

\begin{corollary}
\label{lem:allownonempty}
For all $U \in \beliefsetwin(\game,\target)$ we have $\allow(U) \neq \emptyset$.
\end{corollary}

\begin{figure}[t]
%%\label{fig:p2}
\begin{center}
% \resizebox{12cm}{!}{
\begin{tikzpicture}[>=latex,remember picture,
  inner/.style={circle,draw=blue!50,fill=blue!20,thick,inner sep=2pt},
  outer/.style={draw=black,dashed,fill=green!10,thick,inner sep=2pt}
  ]]
\node[State, solid,initial above, initial text=,xshift=30] (initial) {$s_0$};
   \node[outer] (p1) {
    \begin{tikzpicture}
    \node[State,solid, left of=initial,xshift=-30] (goal) {$T$};
    
    \node[State,solid, right of= initial,xshift=30] (trap) {$B$};
    \draw[->]{
  (initial) edge[solid,bend right] node[above] {$a, \frac{1}{2}$} (goal)
  (initial) edge[solid,bend left] node[above] {$a, \frac{1}{2}$} (trap)
  (initial) edge[solid,bend left] node[below] {$b, \frac{1}{2}$} (goal)
  (initial) edge[solid,loop below] node[right,xshift=1] {$b, \frac{1}{2}$} (goal)
  (goal) edge[solid,loop below] node[left,xshift=-2] {$b, 1$} (goal)
  (goal) edge[solid,loop above] node[left,xshift=-2] {$a, 1$} (goal)
  (trap) edge[solid,loop below] node[right,xshift=2] {$b, 1$} (trape)
  (trap) edge[solid,loop above] node[right,xshift=2] {$a, 1$} (trap)

  };
  \end{tikzpicture}

  };
\node[State, solid,initial above, initial text=,xshift=0] (initial2) {$s_0$};
\end{tikzpicture}

\caption{POMDP $\game$}
\label{fig:example}
\end{center}
\vspace{-1.5em}
\end{figure}

\begin{example}
Consider POMDP $\game$ depicted on Figure~\ref{fig:example} with three states: the initial state $s_0$, and two absorbing states (the target state $T$, and the loosing state $B$). There are two actions $a$ and $b$ available in the initial state, the first action $a$ leads to both the target state and the loosing state, each with probability $1/2$, while the second action~$b$ leads to the initial state and the target state with probability $1/2$ each. This POMDP is not a goal-POMDP, as the target state $T$ is not reachable from the loosing state $B$. Note that belief $\{s_0\}$ belongs to the set $\beliefsetwin(\game,T)$, as the strategy that plays only action $a$ reaches the target state $T$ almost-surely. The set of allowed actions $\allow(\{s_0\})$ does not contain action $b$, as any belief that contains the loosing state $B$ does not belong to the set $\beliefsetwin(\game,T)$.
\end{example}

\smallskip\noindent{\bf Markov chains and reachability.} 
A Markov chain $\ov{\game}=(\ov{S},\ov{\trans})$ consists of a finite set $\ov{S}$ of 
states and a probabilistic transition function $\ov{\trans}:\ov{S} \rightarrow \distr(\ov{S})$.
Given the Markov chain, we consider the directed graph $(\ov{S},\ov{E})$ where $\ov{E}=\set{(\ov{s},\ov{s}') 
\mid \trans(\ov{s}' \mid \ov{s}) >0}$.
% A \emph{recurrent class} $\ov{C} \subseteq \ov{S}$ of the Markov chain is a bottom 
% strongly connected component (scc) in the graph  $(\ov{S},\ov{E})$ 
% (a bottom scc is an scc with no edges out of the scc).
% We denote by $\Rec(\ov{\game})(\ov{s})$ the set of recurrent classes of the Markov chain reachable
% from state $\ov{s}$.
% i.e., $\Rec(\ov{\game})=\set{\ov{C} \mid \ov{C} \text{ is a recurrent class}}$.
% Given a state $\ov{s}$ and a set $\ov{U}$ of states, we say that $\ov{U}$ is
% reachable from $\ov{s}$ if there is a path from $\ov{s}$ to some state in 
% $\ov{U}$ in the graph $(\ov{S},\ov{E})$. 
% Given a state $\ov{s}$ of the Markov chain we denote by $\Rec(\ov{\game})(\ov{s}) \subseteq \Rec(\ov{\game})$ 
% the subset of the recurrent classes reachable from $\ov{s}$ in $\ov{G}$.
% A state is recurrent if it belongs to a recurrent class.
The following standard properties of reachability in Markov chains 
will be used in our proofs~\cite{Kemeny}:
\begin{enumerate}
\item 
\emph{Property~1 of Markov chains.}
For a set $\ov{T} \subseteq \ov{S}$, if for all states $\ov{s} \in \ov{S}$ 
there is a path to $\ov{T}$ (i.e., for all states there is a positive
probability to reach $\ov{T}$), then from all states the set $\ov{T}$
is reached with probability~1.
\item \emph{Property~2 of Markov chains.} In a Markov chain if a set $\ov{T}$ is reached
almost-surely from $\ov{s}$, then the expected hitting time from $\ov{s}$ to
$\ov{T}$ is at most exponential in the number of the states of the Markov chain. 

% \item
% For all states $\ov{s}$, if the Markov chain starts at $\ov{s}$, then 
% the set ${\mathcal C}=\bigcup_{\ov{C} \in \Rec(\ov{\game})(\ov{s})} \ov{C}$ is
% reached with probability~1, 
% i.e., the set of recurrent classes  are reached with probability~1.
% \item 
% For a recurrent class $\ov{C}$, for all states $\ov{s} \in \ov{C}$,
% if the Markov chain starts at $\ov{s}$, then all states
% $\ov{t}\in \ov{C}$ are visited infinitely often with probability~1.
\end{enumerate}

\smallskip\noindent{\bf The strategy $\straa_\allow$.}
We consider a \emph{belief-support-based stationary} (for brevity belief-based)\footnote{recall, for a belief-support-based stationary strategy, 
the probability distribution only depends on the current belief-support} strategy $\straa_\allow$ as follows:
for all belief-supports $U$ in $\beliefsetwin(\game,\target)$, the strategy 
plays uniformly at random all actions from $\allow(U)$. 
Note that as the strategy $\straa_\allow$ is belief-based, it
can be viewed as a finite-memory strategy 
$\straa_\allow = (\straa_u, \straa_n, \beliefsetwin(\game,\target), m_0)$, 
where the components are defined as follows:
(i)~The set of memory elements are the winning belief-supports $\beliefsetwin(\game,\target)$;
(ii)~the belief-support $m_0 \in \beliefsetwin(\game,\target)$ is the initial belief (i.e., $\supp(\initd_0)$);
(iii)~the action selection function given memory $U \in \beliefsetwin(\game,\target)$ is a uniform distribution 
over the set $\allow(U)$ of actions, i.e., 
$\straa_n(U) = \unif(\allow(U))$ where $\unif$ denotes the uniform distribution; and
(iv)~the memory update function given memory $U \in \beliefsetwin(\game,\target)$, observation $z \in \obs$, 
and action $a \in \allow(U)$ is defined as the belief-support update $U'$ from belief-support $U$ under action $a$ and observation $z$,
i.e., $\straa_u(U,z,a)= \update(U,z,a)$.

\smallskip\noindent{\bf The Markov chain $\game \restr \straa_\allow$.}
Given a POMDP $\game=(S,\act,\trans,\obs,\obsmap,s_0)$ and the strategy $\straa_\allow = (\straa_u, \straa_n, \beliefsetwin(\game,\target), m_0)$ the Markov chain $\game \restr \straa_\allow = (\ov{S}, \ov{\trans})$ obtained by
playing strategy $\straa_\allow$ in $\game$ is defined as follows:
\begin{itemize}
\item The set of states $\ov{S}$ is defined as follows: 
$\ov{S} =\set{(s,U) \mid U \in \beliefsetwin(\game,\target), s \in U}$, 
i.e., the second components are the almost-sure winning belief-supports and the first 
component is a state in the belief-support.
\item The probability that the next state is $(s',U')$ from a state $(s,U)$ 
is $\ov{\trans}((s,U))((s',U')) = 
\sum_{a\in\act}\straa_n(U)(a)\cdot\trans(s,a)(s')\cdot\straa_u(U,\obsmap(s'),a)(U')$.
\end{itemize}
The probability of transition can be decomposed as follows:
(i)~First an action $a\in\act$ is sampled according to the distribution $\straa_n(U)$; 
(ii)~then the next state $s'$ is sampled according to the distribution $\trans(s,a)$; and
(iii)~finally the new memory $U'$ is sampled according to the distribution $\straa_u(U,\obsmap(s'),a)$.

\begin{remark}
Note that due to the definition of the strategy $\straa_\allow$ (that only plays allowed actions) 
all states $(s',U')$ of the Markov chain $\game \restr \straa_\allow$ that are reachable from a 
state $(s,U)$ where $s \in U$ and $U \in \beliefsetwin(\game,\target)$ satisfy that $s' \in U'$ and 
$U' \in \beliefsetwin(\game,\target)$. 
%%Therefore, we restrict the states of the chain to states $(s,U)$ that satisfy $s \in U$ and $U \in \beliefsetwin(\game,\target)$.
\end{remark}

\begin{lemma}
\label{lem:allallowwins}
The belief-based strategy $\straa_\allow$ is an almost-sure winning strategy for all belief-supports  
%%for all starting beliefs 
$U \in \beliefsetwin(\game, \target)$ for the objective $\reach(\target)$.
%% starting in belief $U$.
\end{lemma}
\begin{proof}
Consider the Markov chain $\game \restr \straa_\allow$ and a state $(s,U)$ of the Markov chain. 
As $U \in \beliefsetwin(\game,\target)$ by the definition of almost-sure winning belief-supports,
there exists a strategy $\straa$ that is almost-sure winning for the reachability objective
$\reach(\target)$ starting with belief-support $U$.

\smallskip\noindent{\em Reachability under $\straa$.}
Note that by Lemma~\ref{lem:play_allow} the strategy must only play allowed actions.
The strategy must ensure that from $s$ a target state $t$ is reached with positive 
probability by playing according to $\straa$ (given the initial belief-support is $U$).
It follows that there exists a finite prefix of a play $(s_1,a_1,s_2,a_2, \ldots a_{n-1},s_n)$ 
induced by $\straa$ where $s_1 = s$ and $s_n=t$ and for all $1 \leq i <n$ we have that 
$a_i \in \supp(\straa(\obsmap((s_1,a_1,s_2,a_2, \ldots, s_{i}))))$.
We define a sequence of belief-supports $U_1, U_2, \ldots U_n$, where $U_1 = U$ and 
$U_{i+1} = \update(U_i,\obsmap(s_{i+1}),a_i)$.
% $U_{i+1} = \left(\bigcup_{s' \in U_i} \supp(\trans(s',a_i)) \right)\cap \obsmap^{-1}(\obsmap(s_{i+1}))$ for all $1 \leq i <n$. 
As $\straa$ is an almost-sure winning strategy, it follows from Lemma~\ref{lem:play_allow} that $a_{i} \in  \allow(U_i)$ for all $1 \leq i < n$.

\smallskip\noindent{\em Reachability in the Markov chain.}
Recall that the strategy $\straa_\allow$ plays all the allowed actions uniformly at random.
Hence it follows from the definition of the Markov chain $\game\restr \straa_\allow$ that for all $0 \leq i < n$ we 
have $\ov{\trans}((s_i,U_i))((s_{i+1},U_{i+1})) > 0$, i.e,
there is a positive probability to reach $(t, U_n)$ from $(s,U)$ in the Markov chain $\game \restr \straa_\allow$.
It follows that for an arbitrary state $(s,U)$ of the Markov chain $\game \restr \straa_\allow$ there exists a state $(t',U')$ with 
$t' \in \target$ that is reached with positive probability.
In other words, in the graph of the Markov chain, there is a path from all states $(s,U)$ to a state 
$(t',U')$ where $t' \in \target$.
Thus by Property~1 of Markov chains it follows that the target set $\target$ is reached with
probability 1. It follows that $\straa_\allow$ is an almost-sure winning strategy for all
belief-supports in $\beliefsetwin(\game,\target)$.
The desired result follows.
\hfill\qed
\end{proof}

\begin{remark}[Computation of $\straa_\allow$]\label{rem:straa_comp}
It follows from Lemma~\ref{lem:allallowwins} that the strategy $\straa_\allow$ 
can be computed by computing the set of almost-sure winning states in the 
\emph{belief-support MDP}.
The belief-support MDP is a perfect-observation MDP where each state is a belief-support 
of the original POMDP, and given an action, the next state is obtained according 
to the belief-support updates.
The strategy $\straa_\allow$ can be obtained by computing the set of 
almost-sure winning states in the belief-support MDP, and for discrete graph-based 
algorithms to compute almost-sure winning states in perfect-observation MDPs
see~\cite{CY95,CH11}. 
\end{remark}

\noindent{\bf Upper bound.}
We now establish a double-exponential upper bound on $\optCost$, matching our lower bound
from Lemma~\ref{lem:lower_bound}.
We have that $\straa_{\allow} \in \almost_{\game}(\target)$. 
Hence we have $\val(\straa_{\allow}) \geq  \inf_{\straa \in \almost_{\game}(\target)} \val(\straa)= \optCost$.
Once $\straa_{\allow}$ is fixed, since the strategy is belief-based (i.e., depends on the subset of states)
we obtain an exponential size Markov chain. 
%In Markov chains, if a set is reached almost-surely, then the expected hitting time is at most exponential
%in the number of states of the Markov chain~\cite{Kemeny}.
It follows from Property~2 of Markov chains that given $\straa_\allow$ the expected hitting time 
to the target set is at most double exponential.
If $\cost_{\max}$ denotes the maximal cost of transitions, then $\optCost$ is 
bounded by $\cost_{\max}$ times the expected hitting time.
Thus we obtain the following lemma.

\begin{lemma}\label{lemm:doubexp}
Given a POMDP $\game$ with $n$ states, let $\cost_{\max}$ denote the maximal value of the cost of
all transitions.
There is a polynomial function $q$ such that $\optCost \leq 2^{2^{q(n)}} \cdot \cost_{\max}$.
\end{lemma}

\subsection{Optimal finite-horizon strategies}
\label{subsec:optimal}
% Given a POMDP $\game$ with a reachability objective $\reach(\target)$ we will define a strategy $\straa_k$ for $k \in \nat$ as follows: for the first $k$ turns it will play an optimal finite-horizon strategy minimizing the expected cost restricted to allowed actions, after $k$ turns it switches to the almost-sure winning strategy $\straa_\allow$ for the reachability objective.
Our algorithm for approximation of $\optCost$ will use algorithms for optimizing the
finite-horizon costs. 
We first recall the well-known construction of the optimal finite-horizon 
strategies that minimizes the expected total cost in POMDPs for 
length $k$.
%%with a reachability objective.

\smallskip\noindent\textbf{Information state.}
For minimizing the expected total cost, strategies based on information states are sufficient~\cite{sondik}. 
An \emph{information state} $b$ is defined as a probability distribution over the set of states, where for $s \in \states$ 
the value $b(s)$ denotes the probability of being in state $s$. 
We will denote by $\informationstates$ the set of all information states. 
Given an information state $b$, an action $a$, and an observation $z$, computing the resulting information state 
$b'$ can be done in a straight forward way, see~\cite{cassandra1998exact}.

\smallskip\noindent{\bf Value-iteration algorithm.}
The standard finite-horizon value-iteration algorithm for expected total cost in the setting of perfect-information MDPs 
can be formulated by the following equation:
\begin{eqnarray*}
V^*_0(s) &=& 0;\\
V^*_n(s) &=& \min_{a \in \act}\left[ \cost(s,a)+\sum_{s' \in \states}\trans(s,a)(s')V^*_{n-1}(s')\right];
\end{eqnarray*}
where $V^*_n(s)$ represents the value of an optimal policy, when the starting state is $s$ and there are $n$ decision steps remaining.
For a POMDP the finite-horizon value-iteration algorithm works on the information states.
%, one can construct a continuous space perfect-information MDP $H$ where the states are information states $\informationstates$, 
%such that a slight modification of the finite-horizon value-iteration algorithm works. 
Let $\psi(b,a)$ denote the probability distribution over the information states given that action $a$ was 
played in the information state $b$. 
The cost function $\cost':\informationstates\times \act \rightarrow \nat$ that maps every pair of an information state and an action 
to a positive real-valued cost is defined as follows: $\cost'(b,a) = \sum_{s \in \states} b(s) \cdot \cost(s,a)$.
The resulting equation for finite-horizon value-iteration algorithm for POMDPs is as follows:
\begin{eqnarray*}
V^*_0(b) &=& 0;\\
V^*_n(b) &=& \min_{a \in \act}\left[ \cost'(b,a)+\sum_{b' \in \informationstates}\psi(b,a)(b')V^*_{n-1}(b')\right].
\end{eqnarray*}

\smallskip\noindent{\bf The optimal strategy $\straa^\mathsf{FO}_k$ and $\straa^*_k$.}
In our setting we modify the standard finite-horizon value-iteration algorithm by restricting the optimal strategy to play 
only allowed actions and restrict it only to belief-supports in the set $\beliefsetwin(\game,\target)$. 
The equation for the value-iteration algorithm is defined as follows:
\begin{eqnarray*}
V^*_0(b) &=& 0;\\
V^*_n(b) &=& \min_{a \in \allow(\supp(b))}\left[ \cost'(b,a)+\sum_{b' \in \informationstates}\psi(b,a)(b')V^*_{n-1}(b')\right].
\end{eqnarray*}
We obtain a strategy $\straa^\mathsf{FO}_k = (\straa_u, \straa_n, M, m_0)$ that is finite-horizon optimal for length $k$ 
(here FO stands for finite-horizon optimal) from the above equation as follows:
(i)~the set of memory elements $M$ is defined as $\informationstates \times \nat$;
(ii)~the initial memory state is $(\initd_0,k)$;
(iii)~for all $1\leq n \leq k$, the action selection function $\straa_n((b,n))$ selects an arbitrary action $a$ such that $a = \arg\min_{a \in  \allow(\supp(b))}
 \left[ \cost'(b,a)+\sum\limits_{b' \in \informationstates}\psi(b,a)(b')V^*_{n-1}(b') \right]$; and (iv)~the memory update function given 
 a memory state $(b,n)$, action $a$, and an observation $o$ updates to a memory state $(b',n-1)$, where $b'$ is the unique information
 state update from information state $b$ under action $a$ and observation $z$.
As the target states $\target$ in the POMDP $\game$ are absorbing and 
the costs on all outgoing edges from the target states are the only edges with cost $0$, 
it follows that for sufficiently large $n$ the strategy $\straa_\mathsf{FO}^k$ minimizes the expected total cost 
to reach the target set $\target$.
Given $\straa^\mathsf{FO}_k$, we define a strategy $\straa^*_k$ as follows: 
for the first $k$ steps, the strategy $\straa^*_k$ plays as the strategy $\straa^\mathsf{FO}_k$, and 
after the first $k$ steps the strategy plays as the strategy $\straa_{\allow}$.

\begin{lemma}\label{lemm:fin1}
%Assuming $\almost_\game(\target) \not = \emptyset$ we have that 
For all $k \in \nat$ the strategy $\straa^*_k$ is almost-sure winning for the reachability objective $\reach(\target)$.
\end{lemma}
\begin{proof}
By definition the strategy  $\straa^\mathsf{FO}_k$ (and hence the strategy $\straa^*_k$) plays only allowed actions in the first $k$ steps.
Hence it follows that every reachable belief-support in the first $k$ steps belongs to $\beliefsetwin(\game,\target)$. 
After the first $k$ steps, the strategy plays as $\straa_\allow$, and by Lemma~\ref{lem:allallowwins}, 
the strategy $\straa_\allow$ is almost-sure winning for the reachability objective $\reach(\target)$ from every belief-support in 
$\beliefsetwin(\game,\target)$. The result follows.
\hfill\qed
\end{proof}

Note that the only restriction in the construction of the strategy  $\straa^\mathsf{FO}_k$ is that it
must play only allowed actions, and since almost-sure winning strategies only play allowed actions
(by  Lemma~\ref{lem:play_allow}) 
it follows that $\straa^\mathsf{FO}_k$ (and hence $\straa^*_k$) is optimal for the finite-horizon of 
length $k$ (i.e., for the objective $\sumcost_k$) 
among all almost-sure winning strategies.

\begin{lemma}\label{lemm:fin2}
%Assuming $\almost_\game(\target) \not = \emptyset$ we have that 
For all $k \in \nat$ we have 
$\expect_{\initd_0}^{\straa^*_k}[\sumcost_k] = 
\inf_{\straa \in \almost_{\game}(\target)} \expect_{\initd_0}^{\straa}[\sumcost_k] $.
\end{lemma}

Note that since in the first $k$ steps $\straa^*_k$ plays as $\straa^\mathsf{FO}_k$ 
we have the following proposition.

\begin{proposition}\label{prop:fin}
For all $k \in \nat$ we have 
$\expect_{\initd_0}^{\straa^*_k}[\sumcost_k] = 
\expect_{\initd_0}^{\straa^\mathsf{FO}_k}[\sumcost_k]$.
\end{proposition}

\subsection{Approximation algorithm}
In this section we will show that for all $\epsilon>0$ there exists a bound 
$k$ such that the strategy $\straa^*_k$ approximates $\optCost$ within $\epsilon$.
First we consider an upper bound on $\optCost$.

\noindent{\bf Bound $\upper$.}
We consider an upper bound $\upper$ on the expected total cost of the strategy $\straa_\allow$ 
starting in an arbitrary state $s \in U$ with the initial belief-support $U \in \beliefsetwin(\game,\target)$. 
Given a belief-support $U \in \beliefsetwin(\game,\target)$ and a state $s\in U$ let $T_{\allow}(s,U)$ 
denote the expected total cost of the strategy $\straa_\allow$ starting in the state $s$ with 
the initial belief-support $U$. Then the upper bound is defined as 
$\upper = \max_{U \in \beliefsetwin(\game,\target), s \in U} T_{\allow}(s,U)$.
As the strategy $\straa_\allow$ is in $\almost_\game(\target)$ it follows that the value $\upper$ is also 
an upper bound for the optimal cost $\optCost$.
Observe that by Lemma~\ref{lemm:doubexp} it follows that $\upper$ is at most double exponential in the size
of the POMDP.

\begin{lemma}\label{lemm:upper1}
We have $\optCost \leq \upper$.
\end{lemma}

\noindent{\bf Key lemma.} We will now present our key lemma to obtain the bound on $k$ 
depending on $\epsilon$. 
We start with a few notations.
Given $k \in \nat$, let $\cale_k$ denote the event of reaching the target set within 
$k$ steps, i.e., 
$\cale_k= \set{ (s_0, a_0, s_1, a_1, s_2 \ldots) \in \Omega \mid \exists i \leq k:  s_i \in \target}$;
and $\ov{\cale}_k$ the complement of the event $\cale_k$ that denotes the target set is 
not reached within the first $k$ steps.
Recall that for plays $\rho = (s_0,a_0,s_1,a_1,s_2,a_2,\ldots)$ we have 
$\sumcost_k = \sum_{i=0}^{k}\cost(s_i,a_i)$ and we consider
$\ov{\sumcost}_k = \sum_{i=k+1}^\infty \cost(s_i,a_i)$ the sum of the costs after $k$ steps.
Note that we have $\sumcost= \sumcost_k + \ov{\sumcost}_k$.

\begin{lemma}\label{lemm:upper2}
For $k \in \nat$ consider the strategy $\straa^*_k$ that is obtained by playing 
an optimal finite-horizon strategy $\straa^\mathsf{FO}_k$ for $k$ steps, followed
by strategy $\straa_{\allow}$.
Let $\alpha_k = \prb_{\initd_0}^{\straa^*_k}(\ov{\cale}_k)$ denote the probability 
that the target set is not reached within the first $k$ steps.
We have 
\[
\expect_{\initd_0}^{\straa^*_k}[\sumcost] 
\leq \expect_{\initd_0}^{\straa^*_k}[\sumcost_k] + \alpha_k \cdot \upper 
%= 
%\expect_{\initd_0}^{\straa^\mathsf{FO}_k}[\sumcost_k] + \alpha_k \cdot \upper  
\]
\end{lemma}
\begin{proof}
We have 
\[
\begin{array}{rcl}
\expect_{\initd_0}^{\straa^*_k}[\sumcost] 
& = &
\prb_{\initd_0}^{\straa^*_k}(\cale_k) \cdot \expect_{\initd_0}^{\straa^*_k}[\sumcost \mid \cale_k] 
+ 
\prb_{\initd_0}^{\straa^*_k}(\ov{\cale}_k) \cdot \expect_{\initd_0}^{\straa^*_k}[\sumcost \mid \ov{\cale}_k] \\[2ex]
& = & 
\prb_{\initd_0}^{\straa^*_k}(\cale_k) \cdot \expect_{\initd_0}^{\straa^*_k}[\sumcost \mid \cale_k] 
+ 
\prb_{\initd_0}^{\straa^*_k}(\ov{\cale}_k) \cdot \expect_{\initd_0}^{\straa^*_k}[(\sumcost_k + \ov{\sumcost}_k) \mid \ov{\cale}_k] 
\\[2ex]
& = & 
\prb_{\initd_0}^{\straa^*_k}(\cale_k) \cdot \expect_{\initd_0}^{\straa^*_k}[\sumcost \mid \cale_k] 
+ 
\prb_{\initd_0}^{\straa^*_k}(\ov{\cale}_k) \cdot \expect_{\initd_0}^{\straa^*_k}[\sumcost_k \mid \ov{\cale}_k ] +
\prb_{\initd_0}^{\straa^*_k}(\ov{\cale}_k) \cdot \expect_{\initd_0}^{\straa^*_k}[\ov{\sumcost}_k \mid \ov{\cale}_k] \\[2ex]
& = & 
\prb_{\initd_0}^{\straa^*_k}(\cale_k) \cdot \expect_{\initd_0}^{\straa^*_k}[\sumcost_k \mid \cale_k] 
+ 
\prb_{\initd_0}^{\straa^*_k}(\ov{\cale}_k) \cdot \expect_{\initd_0}^{\straa^*_k}[\sumcost_k \mid \ov{\cale}_k ] +
\prb_{\initd_0}^{\straa^*_k}(\ov{\cale}_k) \cdot \expect_{\initd_0}^{\straa^*_k}[\ov{\sumcost}_k \mid \ov{\cale}_k] \\[2ex] 
& = & 
\expect_{\initd_0}^{\straa^*_k}[\sumcost_k] + 
\prb_{\initd_0}^{\straa^*_k}(\ov{\cale}_k) \cdot \expect_{\initd_0}^{\straa^*_k}[\ov{\sumcost}_k \mid \ov{\cale}_k]
\leq 
\expect_{\initd_0}^{\straa^*_k}[\sumcost_k] + \alpha_k \cdot \upper.
\end{array}
\]
The first equality is obtained by splitting with respect to the complementary events $\cale_k$ and $\ov{\cale}_k$;
the second equality is obtained by writing $\sumcost= \sumcost_k + \ov{\sumcost}_k$; and
the third equality is by linearity of expectation. 
The fourth equality is obtained as follows: since all outgoing transitions from target states have cost 
zero, it follows that given the event $\cale_k$ we have $\sumcost=\sumcost_k$.
The fifth equality is obtained by combining the first two terms.
The final inequality is obtained as follows: from the $(k+1)$-th step the 
strategy plays as $\straa_{\allow}$ and the expected total cost given 
$\straa_{\allow}$ is bounded by $\upper$.
The result follows.
\hfill\qed
\end{proof}

\begin{comment}
\begin{corollary}\label{coro:upper}
For $k \in \nat$ consider the strategy $\straa^*_k$ and $\alpha_k$ 
(as defined in Lemma~\ref{lemm:upper2}).
We have 
\[
\expect_{\initd_0}^{\straa^*_k}[\sumcost] \leq \optCost + \alpha_k \cdot \upper
\]
\end{corollary}
\begin{proof}
For all $k \in \nat$ we have 
\[
\begin{array}{rcl}
\expect_{\initd_0}^{\straa^*_k}[\sumcost] & \leq & 
\expect_{\initd_0}^{\straa^*_k}[\sumcost_k] + \alpha_k \cdot \upper \\[2ex]
& \leq &  
(\inf_{\straa \in \almost_{\game}(\target)} \expect_{\initd_0}^{\straa}[\sumcost_k]) + \alpha_k \cdot \upper \\[2ex]
& \leq &  
(\inf_{\straa \in \almost_{\game}(\target)} \expect_{\initd_0}^{\straa}[\sumcost]) + \alpha_k \cdot \upper 
= \optCost + \alpha_k \cdot \upper.
\end{array}
\]
The first inequality follows from Lemma~\ref{lemm:upper2};
the second inequality follows from Lemma~\ref{lemm:fin2}; 
the third inequality follows from the fact that $\sumcost_k \leq \sumcost$ 
for non-negative weights; and the last equality follows from definition 
of $\optCost$.
The desired result follows.
\hfill\qed
\end{proof}
\end{comment}

\begin{lemma}\label{lemm:upper3}
For $k \in \nat$ consider the strategy $\straa^*_k$ and $\alpha_k$ 
(as defined in Lemma~\ref{lemm:upper2}).
The following assertions hold:
\[
(1) \ 
\expect_{\initd_0}^{\straa^*_k}[\sumcost_k] \leq \optCost; \quad 
\text{and} \quad 
(2)\ \alpha_k \leq \frac{\optCost}{k}.
\]
\end{lemma}
\begin{proof}
We prove both the inequalities below.
\begin{enumerate}
\item For $k \in \nat$ we have that
\[
\expect_{\initd_0}^{\straa^*_k}[\sumcost_k] \leq 
\inf_{\straa \in \almost_{\game}(\target)} \expect_{\initd_0}^{\straa}[\sumcost_k]
\leq \inf_{\straa \in \almost_{\game}(\target)} \expect_{\initd_0}^{\straa}[\sumcost] =\optCost
\]
The first inequality is due to Lemma~\ref{lemm:fin2} and the second inequality 
follows from the fact that $\sumcost_k \leq \sumcost$ 
for non-negative weights.
%Finally the fact that $\straa^*_k$ is the same as ${\straa^\mathsf{FO}_k}$ gives us the desired result.

\item Note that $\alpha_k$ denotes the probability that the target state is not reached within the first
$k$ steps. 
Since the costs are positive integers (for all transitions other than the target state transitions), 
given the event $\ov{\cale}_k$ the total cost for $k$ steps is at least $k$.
Hence it follows that $\expect_{\initd_0}^{\straa^*_k}[\sumcost_k] \geq k \cdot \alpha_k$.
Thus it follows from the first inequality that we have $\alpha_k \leq \frac{\optCost}{k}$.
\end{enumerate}
The desired result follows.
\hfill\qed
\end{proof}

\noindent{\bf Approximation algorithms.} Our approximation algorithm is presented as 
Algorithm~\ref{alg:Approx}.
\begin{comment} 
% START To be removed:
The algorithms for approximation are as follows: initialize $k:=1$.
\begin{enumerate}
\item Step~1: Compute the optimal strategy $\straa^\mathsf{FO}_k$ for $k$ steps along 
with $T_k:=\expect_{\initd_0}^{\straa^\mathsf{FO}_k}[\sumcost_k]$ and 
$\alpha_k: =  \prb_{\initd_0}^{\straa^*_k}(\ov{\cale}_k) :=  
\prb_{\initd_0}^{\straa^\mathsf{FO}_k}(\ov{\cale}_k)$.

\item Stopping criteria: 
(a)~\emph{(Additive approximation):} If $\alpha_k \cdot \upper \leq \epsilon$, then stop.
%for additive approximation
%and return the strategy $\straa^*_k$ obtained by  $\straa^\mathsf{FO}_k$ for $k$-steps
%followed by $\straa_{\allow}$;
(b)~\emph{(Multiplicative approximation):} 
If $\alpha_k \cdot \upper \leq T_k \cdot \epsilon$, then stop. 
%for multiplicative approximation
If the algorithm stops, then return the strategy $\straa^*_k$ obtained by playing 
$\straa^\mathsf{FO}_k$ for $k$-steps followed by $\straa_{\allow}$; 
otherwise increment $k$ and goto Step~1.
%% with $k+1$.
\end{enumerate}
We refer the algorithms as {\sc ApproxAlgo}.
% END To be removed:
\end{comment}

\begin{algorithm}[h!]
  \caption{{\sc ApproxAlgo} \textbf{Input:} POMDP, $\epsilon >0 $}
  \label{alg:Approx}
\begin{algorithmic}[1]
% \State \textbf{Input:} $\epsilon$.
\State $k \gets 1$
\State $\straa_\allow, \upper \gets$ Compute $\straa_\allow$ and $\upper$ \Comment See Remark~\ref{rem:straa_comp}
\State $\straa^\mathsf{FO}_k \gets $ Finite-horizon value iteration for horizon $k$
 restricted to allowed actions \Comment See Section~\ref{subsec:optimal}
\State $T_k \gets \expect_{\initd_0}^{\straa^\mathsf{FO}_k}[\sumcost_k]$
\State $\alpha_k \gets \prb_{\initd_0}^{\straa^\mathsf{FO}_k}(\ov{\cale}_k) $  \Comment Note that $\prb_{\initd_0}^{\straa^*_k}(\ov{\cale}_k)=  
\prb_{\initd_0}^{\straa^\mathsf{FO}_k}(\ov{\cale}_k)$
\State Add. approx.: \textbf{if} {$\alpha_k \cdot \upper \leq \epsilon$} \textbf{then goto line:} $10$
\State Mult. approx.: \textbf{if} {$\alpha_k \cdot \upper \leq T_k \cdot \epsilon$} \textbf{then goto line:} $10$
\State $k \gets k+1$
\State \textbf{goto line:} $3$
\State \textbf{return} Strategy $\straa^*_k$ obtained by playing 
$\straa^\mathsf{FO}_k$ for $k$-steps followed by $\straa_{\allow}$
\end{algorithmic}
\end{algorithm}

\smallskip\noindent{\bf Correctness and bound on iterations.}
Observe that by Proposition~\ref{prop:fin} we have 
%at $\straa^*_k$ plays as $\straa^\mathsf{FO}_k$ for the first $k$ steps we have 
$T_k
=\expect_{\initd_0}^{\straa^\mathsf{FO}_k}[\sumcost_k]
=\expect_{\initd_0}^{\straa^*}[\sumcost_k]$.
Thus by Lemma~\ref{lemm:upper2} and Lemma~\ref{lemm:upper3} (first inequality) 
we have 
$\expect_{\initd_0}^{\straa^*}[\sumcost] \leq T_k + \alpha_k \cdot \upper \leq \optCost + 
\alpha_k \cdot \upper$ (since $T_k \leq \optCost$).
Thus if $\alpha_k \cdot \upper \leq \epsilon$ we obtain an additive approximation. 
If $\alpha_k \cdot \upper \leq T_k \cdot \epsilon$, then 
we have  
$\expect_{\initd_0}^{\straa^*}[\sumcost] \leq T_k + \alpha_k \cdot \upper \leq T_k\cdot (1+\epsilon)
\leq \optCost\cdot(1+\epsilon)$; and we obtain an multiplicative approximation.
This establishes the correctness Algorithm~\ref{alg:Approx}.
Finally, we present the theoretical upper bound on $k$ that ensures stopping for 
the algorithm.
By Lemma~\ref{lemm:upper3} (second inequality) we have $\alpha_k \leq \frac{\optCost}{k} \leq \frac{\upper}{k}$.
Thus $k \geq \frac{\upper^2}{\epsilon}$ ensures that $\alpha_k \cdot \upper \leq \epsilon$, and 
the algorithm stops both for additive and multiplicative approximation.  
We now summarize the main results of this section.

\begin{theorem}
In POMDPs with positive costs, the additive and multiplicative approximation problems for the 
optimal cost $\optCost$ are decidable.
Algorithm~\ref{alg:Approx} computes the approximations using finite-horizon optimal strategy computations
and requires at most double-exponentially many iterations; and there exists POMDPs where 
double-exponentially many iterations are required.
%%can be achieved by running the fini
%If $\almost_\game(\target) \not = \emptyset$, then the strategy $\straa_{\mathsf{result}}$ belongs to the set $\almost_\game(\target)$ and is minimizing the expected total cumulative sum within a distance of $\epsilon$ within the optimal cost $\optCost$. The expected number of steps before reaching the target set is exponential in the size of the POMDP.
\end{theorem}

\begin{remark}
Note that though the theoretical upper bound $k$ on the number of iterations 
$\frac{\upper^2}{\epsilon}$ is double exponential in the worst case, in practical
examples of interest the stopping criteria is expected to be satisfied in 
much fewer iterations.
%
%We expect the probability of not reaching the target set $\target$ with the optimal finite-horizon strategy to decrease faster in real-life examples. Therefore our implementation keeps track of the probability of reaching the target set and whenever the probability is sufficiently high it switches to the strategy $\straa_{\allow}$ that ensures reaching the target set almost-surely.
\end{remark}

\begin{remark}\label{rem:almostsure}
We remark that if we consider POMDPs with positive costs, 
then considering almost-sure strategies is not a restriction.
For every strategy that is not almost-sure winning, with positive
probability the target set is not reached, and since all costs
are positive, the expected cost is infinite. 
If every strategy is not almost-sure winning (i.e., there exists no almost-sure 
winning strategy to reach the target set from the starting state), 
then the expected cost is infinite, and if there exists an almost-sure winning strategy, 
we approximate the optimal cost. 
Thus our result is applicable to \emph{all} POMDPs with positive costs. 
A closely related work to ours is Goal-POMDPs, and the solution of Goal-POMDPs 
applies to the class of POMDPs where the target state is reachable from every 
state (see~\cite[line-3, right column page~1]{BG09} for the restriction of Goal-MDPs
and the solution of Goal-POMDPs is reduced to Goal-MDPs).
For example, in the following Section~\ref{sec:experiment}, the first three examples for 
experimental results do not satisfy the restriction of Goal-POMDPs.
\end{remark}

\newcommand{\goodObs}{\mathit{Good}}
\newcommand{\badObs}{\mathit{Bad}}
\newcommand{\pos}{\mathit{Position}}
\newcommand{\rockType}{\mathit{RockType}}
\newcommand{\sample}{\mathit{sample}}
\newcommand{\checkAction}{\mathit{check}}

\section{Experimental Results}\label{sec:experiment}

%%%% PARAGRAPH FOR SHORT VERSION
\noindent{\em Implementation.}
We have implemented Algorithm~\ref{alg:Approx}. 
In principle our algorithm suggests the following: 
First, compute the almost-sure winning belief-supports,  
the set of allowed actions, and $\straa_{\allow}$; and 
then compute finite-horizon value iteration restricted to 
allowed actions. 
An important feature of our algorithm is its flexibility that 
any finite-horizon value iteration algorithm can be used for 
our purpose.
We have implemented our approach where we first implement the 
computation of almost-sure winning belief-supports, allowed actions, and 
$\straa_{\allow}$; 
and for the finite-horizon value iteration 
(Step~3 of Algorithm~\ref{alg:Approx}) we implement two approaches.
The first approach is the exact finite-horizon value iteration using 
a modified version of POMDP-Solve~\cite{POMDPSolve};
and the second approach is an approximate finite-horizon value iteration 
using a modified version of RTDP-Bel~\cite{BG09};
and in both the cases our straightforward modification is that the computation 
of the finite-horizon value iteration is restricted to allowed actions and 
almost-sure winning belief-supports.

\smallskip\noindent{\em Examples for experimental results.}
We experimented on several well-known examples of POMDPs. 
The POMDP examples we considered are as follows:
(A)~We experimented with the \emph{Cheese maze} POMDP example which was introduced in~\cite{C92} 
and also studied in~\cite{D00,LCK95,MB05}. 
Along with the standard example, we also considered a larger maze version; and 
considered two cost functions: one that assign cost~1 to all transitions and
the other where the cost of movement on the baseline is assigned cost~2.
(B)~We considered the \emph{Grid} POMDP introduced in~\cite{Rbook95} and 
also studied in~\cite{LCK95,PR95,MB05}. We considered two cost functions:
one where all costs are~1 and the other where transitions in narrow
areas are assigned cost~2.
(C)~We experimented with the robot navigation problem POMDP introduced in~\cite{LCK95}, where
we considered both deterministic transition and a randomized version. We also 
considered two cost functions: one where all costs are assigned~1 and the other where
costs of turning is assigned cost~2.
(D)~We consider the Hallway example from~\cite{LCK95,S04,SS04,BG09}.
(E)~We consider the RockSample example from~\cite{BG09,SS04}.

\smallskip\noindent{\em Discussion on Experimental results.}
Our experimental results are shown in Table~\ref{tab:results}, where we compare
our approach to RTDP-Bel~\cite{BG09}. 
Other approaches such as SARSOP~\cite{KHL08}, anytime POMDP~\cite{PGT03}, ZMDP~\cite{SS04} are 
for discounted setting, and hence are different from our approach.
The RTDP-Bel approach works only for Goal-POMDPs where from every state the goal states
are reachable, and our first five examples do not fall into this category. 
For the first three examples, both of our exact and approximate implementation 
work very efficiently. 
For the other two larger examples, the exact method does not work since POMDP-Solve cannot 
handle large POMDPs, whereas our approximate method gives comparable result to 
RTDP-Bel.
For the exact computation, we consider multiplicative approximation with $\epsilon=0.1$
and report the number of iterations and the time required by the exact computation.
For the approximate computation, we report the time required by the number of trials
specified for the computation of the finite-horizon value iteration.
For the first three examples, the obtained value of the strategies of our approximate 
version closely matches the value of the strategy of the exact computation, and 
for the last two examples, the values of the strategies obtained by our approximate
version closely matches the values of the strategies obtained by RTDP-Bel.
%%% Further details about the examples are in the Supplementary material.
%%%

\begin{table*}[t]
\centering
\resizebox{\linewidth}{!}{
\begin{tabular}{|c|c|c|c|ccc|ccc|ccc|}
\toprule
\multirow{2}{*}{Example} & \multirow{2}{*}{\:Costs\:} & \multirow{2}{*}{\: $\vert \states \vert$, $\vert \act \vert$, $\vert \obs \vert$\:} & \multirow{2}{*}{\: $\straa_\allow$ comp.} & \multicolumn{3}{c|}{\textbf{Exact} $\epsilon = 0.1$} & \multicolumn{3}{c|}{\textbf{Approx.}} & \multicolumn{3}{c|}{\textbf{RTDP-Bel}} \\
% \cmidrule{5-8}
& &   & &  Iter. & Time & Val.  & Time & Trials & Val.& Time & Trials & Val.\\

% & & &  & \:Iter. & Time &\: Iter. & Time \\
\midrule
\multirow{2}{*}{Cheese maze - small}& $\set{1}$ & \multirow{2}{*}{12, 4, 8}   & \multirow{2}{*}{$0.27 \cdot 10^{-3} $s}    & 7 & 0.54s &   4.6         & 0.06s & 12k  &  4.6 & \multicolumn{3}{c|}{$\times$}            \\
& $\set{1,2}$ &  &      & 8 & 0.62s  &    7.2      &  0.06s & 12k & 7.2 & \multicolumn{3}{c|}{$\times$}              \\
\cmidrule{2-13}
\multirow{2}{*}{Cheese maze - large}& $\set{1}$ & \multirow{2}{*}{16, 4, 8}  & \multirow{2}{*}{$0.57 \cdot 10^{-3}$s}    & 9 & 12.18s &    6.4      & 0.29s & 12k & 6.4 & \multicolumn{3}{c|}{$\times$}              \\
& $\set{1,2}$ &       &    & 12 & 16.55s &   10.8      & 0.3s& 12k & 10.8 & \multicolumn{3}{c|}{$\times$}               \\
\cmidrule{2-13}
\multirow{2}{*}{Grid} & $\set{1}$ & \multirow{2}{*}{11, 4, 6}               & \multirow{2}{*}{$0.47 \cdot 10^{-3}$s}   & 6 & 0.33s & 3.18 & 0.2s & 12k & 3.68 & \multicolumn{3}{c|}{$\times$}               \\
& $\set{1,2}$ & & & 10 & 4.21s & 5.37 & 0.21s & 12k & 5.99 & \multicolumn{3}{c|}{$\times$}            \\
\cmidrule{2-13}
\multirow{2}{*}{Robot movement - det.} & $\set{1}$ & \multirow{2}{*}{15, 3, 11}  & \multirow{2}{*}{$0.43 \cdot 10^{-3}$s}                & 9 & 5.67s  & 7.0 & 0.08s & 12k & 7.0& \multicolumn{3}{c|}{$\times$}                \\
                                    & $\set{1,2}$ &    &                 & 8 & 5.01s & 10.0 & 0.08s & 12k & 10.0 & \multicolumn{3}{c|}{$\times$}             \\
\cmidrule{2-13}
\multirow{2}{*}{Robot movement - ran.} & $\set{1}$ & \multirow{2}{*}{15, 3, 11}& \multirow{2}{*}{$0.52 \cdot 10^{-3}$s}              & 10 & 6.64s & 7.25 & 0.08s& 12k & 7.25& \multicolumn{3}{c|}{$\times$}                 \\
                                    & $\set{1,2}$ &    &               & 10 & 6.65s  & 10.35 & 0.04s & 12k & 10.38 &    \multicolumn{3}{c|}{$\times$}           \\
\cmidrule{2-13}
Hallway & $\set{1}$ & 61, 5, 22  & $0.32 \cdot 10^{-1}$s        & \multicolumn{3}{c|}{Timeout 20m.} & 283.88s& 12k & 6.09 & 282.47s & 12k &   6.26              \\
\cmidrule{2-13}
Hallway 2& $\set{1}$ & 94, 5, 17 & $0.58 \cdot 10^{-1}$s        & \multicolumn{3}{c|}{Timeout 20m.} & 414.29s& 14k & 4.69 & 413.21s & 14k &  4.46              \\
\cmidrule{2-13}
RockSample[4,4] & $\set{1,50,100}$ & 257, 9, 2   & $0.05 $s       & \multicolumn{3}{c|}{Timeout 20m.} & 61.23s& 20k & 542.49 & 61.29s & 20k & 546.73                \\
\cmidrule{2-13}
RockSample[5,5] & $\set{1,50,100}$ & 801, 10, 2 & $0.26$s       & \multicolumn{3}{c|}{Timeout 20m.} & 99.13s& 20k & 159.39 & 98.44s & 20k &  161.07              \\
\cmidrule{2-13}
RockSample[5,7] & $\set{1,50,100}$ & 3201, 12, 2   & $4.44 $s       & \multicolumn{3}{c|}{Timeout 20m.} & 427.94s& 20k & 6.02 & 422.61s & 20k & 6.14           \\
\cmidrule{2-13}
RockSample[7,8] & $\set{1,50,100}$ & 12545, 13, 2   & $78.83 $s       & \multicolumn{3}{c|}{Timeout 20m.} & 1106.2s& 20k & 6.31 &  1104.53s & 20k  & 6.39            \\

\midrule
\end{tabular}
}
\caption{Experimental results}
\vspace{-0.5em}
\label{tab:results}
\end{table*}

\smallskip\noindent{\bf Details of the POMDP examples.} 
We now present the details of the POMDP examples.

\begin{enumerate}

\item \emph{Cheese maze:}
%We have considered two variants of the \emph{Cheese maze} POMDP that differ
%in size, e.g., a small and a large Cheese maze. The smaller variant of the POMDP, shown in Figure~\ref{fig:cheese_maze}, was first introduced 
%in~\cite{C92} and then used also in the work of~\cite{D00,LCK95,MB05}. 
The example models a simple maze, where there are four actions {\tt n}, {\tt e}, {\tt s}, {\tt w} that correspond 
to the movement in the four compass directions. 
The POMDP examples are shown in Figure~\ref{fig:cheese_maze_small} and 
Figure~\ref{fig:cheese_maze_large}.
Actions that attempt to move outside of the maze have no effect on the position; otherwise the 
movement is determined deterministically given the action in all four directions. 
In the small version, there are $12$ states and $8$ observations, which correspond to what walls would be seen 
in all four directions that are immediately adjacent to the current location, i.e., states $5,6,$ and $7$ have
the same observation. 
The game starts in a unique initial state that is not depicted in the figure, where all actions lead
to the baseline states  $0$, $1$, $2$, $3$, or $4$ with uniform probability. 
The target state is depicted with a star, and there are also two absorbing trap states depicted with a skull. 
The initial state, the trap states, and the goal state have their own unique observations. 
In the larger variant of the POMDP there are four more states and intuitively
they add a new leftmost branch to the POMDP with a third absorbing trap state at the end. 
The new baseline is formed out of states $0,1,2, \ldots, 6$. 
In the first setting all the costs are $1$, and this represents the number of steps to the 
target state; and in the second setting the cost of any movement on the baseline is $2$ and 
the movement in the branches costs $1$ (which models that baseline exploration is more costly).
% Intuitively, this setting makes the exploration
% of the baseline more costly.

\begin{figure}[ht]
  \begin{minipage}[b]{0.40\linewidth} % A minipage that covers half the page
    \centering
        \includegraphics[width=5.1cm]{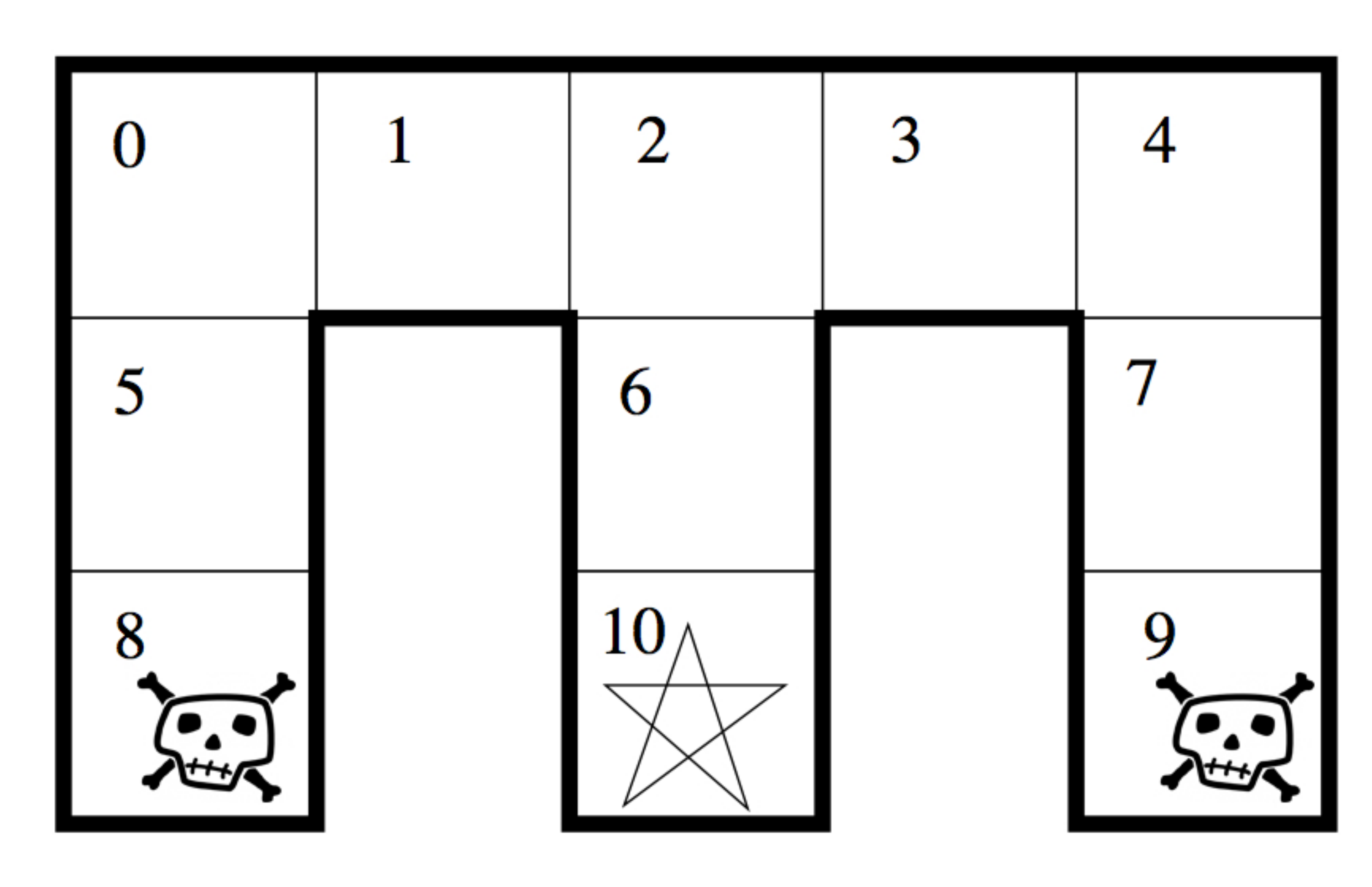}
   \caption{The Cheese Maze - small POMDP }
   \label{fig:cheese_maze_small}
  \end{minipage}
  % \hspace{1.5cm}
  \begin{minipage}[b]{0.54\linewidth}
    \centering
 \centering
        \includegraphics[width=7cm]{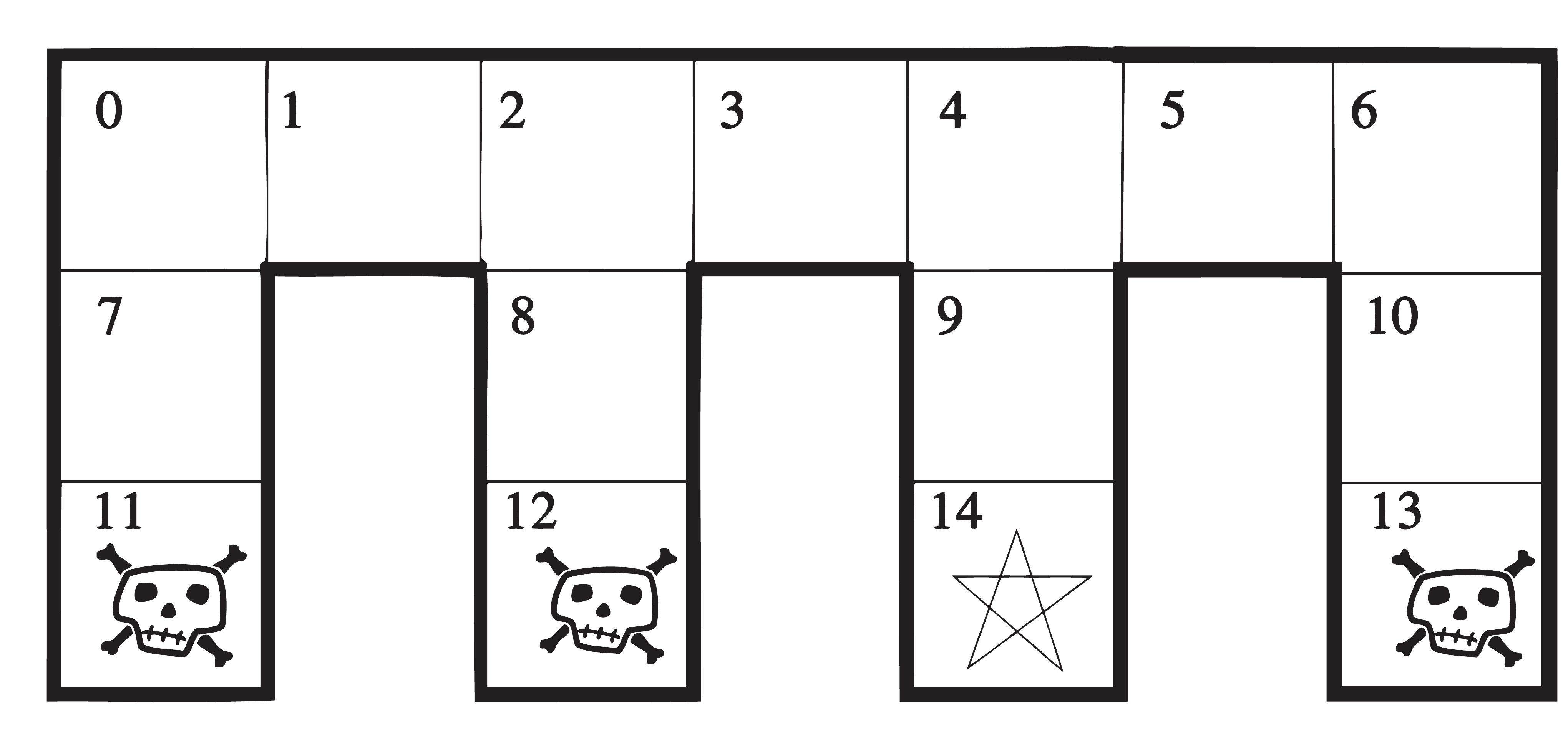}
   \caption{The Cheese Maze - large POMDP }
   \label{fig:cheese_maze_large}
  \end{minipage}
\end{figure}

%\begin{figure}
%\begin{center}
%\includegraphics[width=7cm]{figures/cheeseMaze.pdf}
%\caption{The Cheese Maze - small POMDP}
%\label{fig:cheese_maze}
%\end{center}
%\end{figure}

\item \emph{Grid $4\times3$:}
%%The \emph{Grid} POMDP was described for the first time in~\cite{Rbook95} and then also used in the work of~\cite{LCK95,PR95,MB05}.
The Grid POMDP is shown in Figure~\ref{fig:grid} and models a maze with $11$ states: the starting state $0$, one target 
state depicted with a star, and an absorbing trap state that is depicted with a skull. There are 
four actions {\tt n}, {\tt e}, {\tt s}, {\tt w} that correspond to the movement in the four compass directions. 
The movement succeeds only with probability $0.96$ and with probability $0.02$ moves perpendicular to the intended direction. 
Attempts to move outside of the grid have no effect on the position, i.e., playing action {\tt s} from state $0$ will move 
with probability $0.96$ to state $4$; with probability $0.02$ to state $1$, and with probability $0.02$ to state $0$.
There are $6$ observations that correspond to the information from  detectors that can detect whether there are  walls
immediately adjacent to the east and to the  west of the current state. The goal and the absorbing trap state have their 
own observations.
We have again considered the setting where all the costs before reaching the target state are $1$. 
In the second setting we have assigned to movements in the narrow areas of the maze (states $0,1,4,7$, and $8$) cost of all actions to $2$. 
Intuitively, the higher costs compensate for the wall bumps that can make the movement in the 
narrow areas of the maze more predictable. 
\begin{figure}
\begin{center}
\includegraphics[width=7cm]{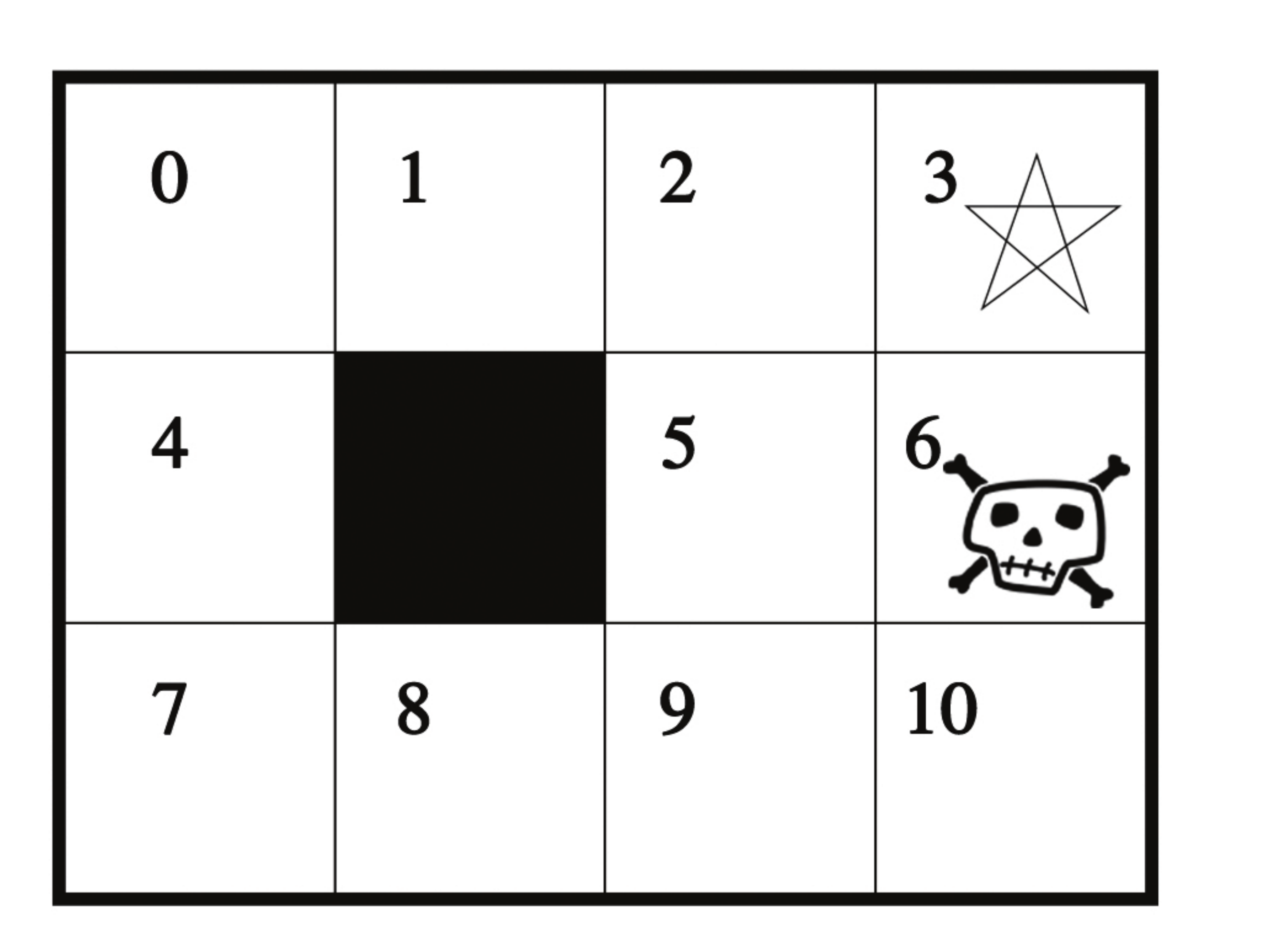}
\caption{The Grid $4\times3$ POMDP}
\label{fig:grid}
\end{center}
\end{figure}

\item \emph{Robot navigation:}
%The robot navigation problem POMDP was introduced in~\cite{LCK95} and is shown in Figure~\ref{fig:robot}. 
The robot navigation POMDP models the movement of a robot in an environment. 
The robot can be in four possible states: facing north, east, south, and west. 
The environment has states $1$, $2$, $3$, and a final state depicted with a star. 
The robot has three available actions: move forward {\tt f}, turn left {\tt l}, and turn right {\tt r}. 
The original setting of the problem is that all actions are deterministic -- \emph{Robot movement - det.} 
We also consider a variant \emph{Robot movement - ran.}, where the attempt to make an action may fail and
with probability $0.04$ has no effect, i.e., the action does not change the state. 
The POMDP starts in a unique initial state that is not depicted in the figure and
under all actions reaches the state $1$ with the robot facing north, east, south or west with uniform probability. 
Any bump to the wall results in a damaged immobile robot, modeled by an absorbing state not depicted in the figure. 
There are 11 observations that correspond to what would be seen in all four directions that are adjacent to the current location. 
The initial state, the damaged state, and the target state have their own observations.
For both variants we have considered two different cost settings. 
In the first setting all the costs before reaching the target state are $1$. 
In the second setting we assign cost $1$ to the move forward action, and 
cost $2$ to the turn-left and turn-right action
(i.e., turning is more costly than moving forward).

\begin{figure}
\begin{center}
\includegraphics[width=3cm]{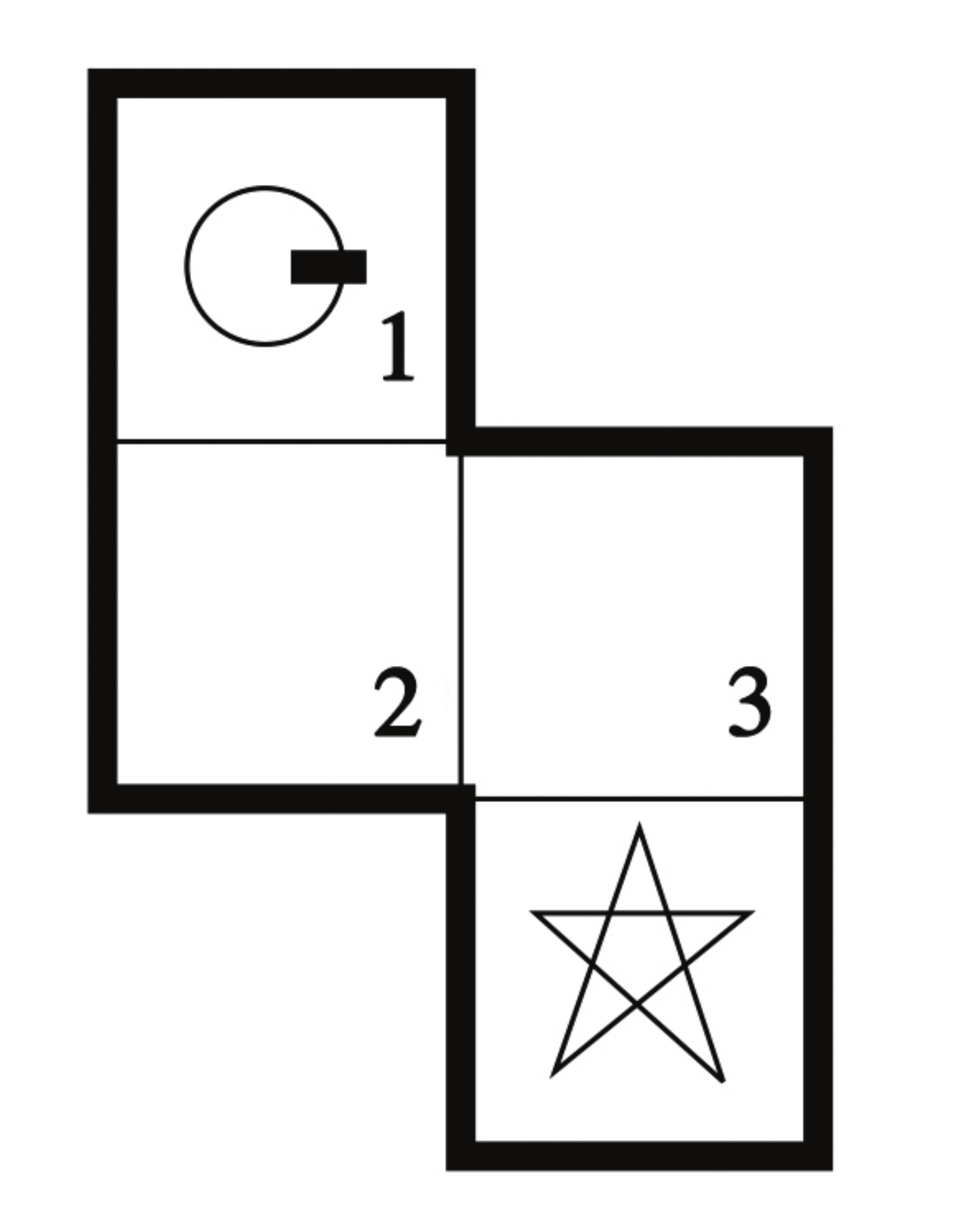}
\caption{The Robot POMDP}
\label{fig:robot}
\end{center}
\end{figure}

\item \emph{Hallway.}
We consider two versions of the Hallway example introduced in in~\cite{LCK95} and used later in~\cite{S04,SS04,BG09}.
The basic idea behind both of the Hallway problems, is that there is an agent 
wandering around some office building. 
% The problems were introduced in~\cite{LCK95} . 
It is assumed that the locations have been discretized so there are a finite number of 
locations where the agent could be. The agent has a small finite set of actions it can take, 
but these only succeed with some probability. Additionally, the agent is equipped 
with very short range sensors to provide it only with information about whether it is 
adjacent to a wall. These sensors also have the property that 
they are somewhat unreliable and will sometimes miss a wall or see a wall when
there is none. It can ''see'' in four directions: forward, left, right,  and backward. It is important to note that these observations are relative to the 
current orientation of the agent (N, E, S, W). In these problems the location in the building and the agent's current orientation comprise the states. There is a single goal location, denoted by the star. The actions that can be chosen consists of movements: forward, turn-left, turn-right, turn-around, and no-op
(stay in place). Action forward succeeds with probability $0.8$, leaves the state unchanged with probability $0.05$, moves the agent to the left and rotates the agent to the left with probability $0.05$, similarly with probability $0.05$ the agent moves to the right and is rotated to the right,
with probability $0.025$ the agent is moved back without changing its orientation, and with probability $0.025$ the agent is moved back and is rotated backwards. The action move-left and move-right succeeds with probability $0.7$, and with probability $0.1$ each of the three remaining orientation is reached. Action turn-around succeeds with probability $0.6$,
leaves the state unchanged with probability $0.1$, turns the agent to left or right, each with probability $0.15$. 
The last action no-op leaves the state unchanged with probability $1$. In states where moving forward is impossible the probability mass 
for the impossible next state is collapsed into the probability of not changing the state.
Every move of the agent has a cost of $1$ and the agent starts with uniform probability in all non-goal states. In the smaller Hallway problem there are $61$ states and $22$ observations. In the Hallway2 POMDP there are $94$ states and $17$ observations.

\begin{figure}[h]
  \begin{minipage}[b]{0.40\linewidth} % A minipage that covers half the page
    \centering
        \includegraphics[width=7.5cm]{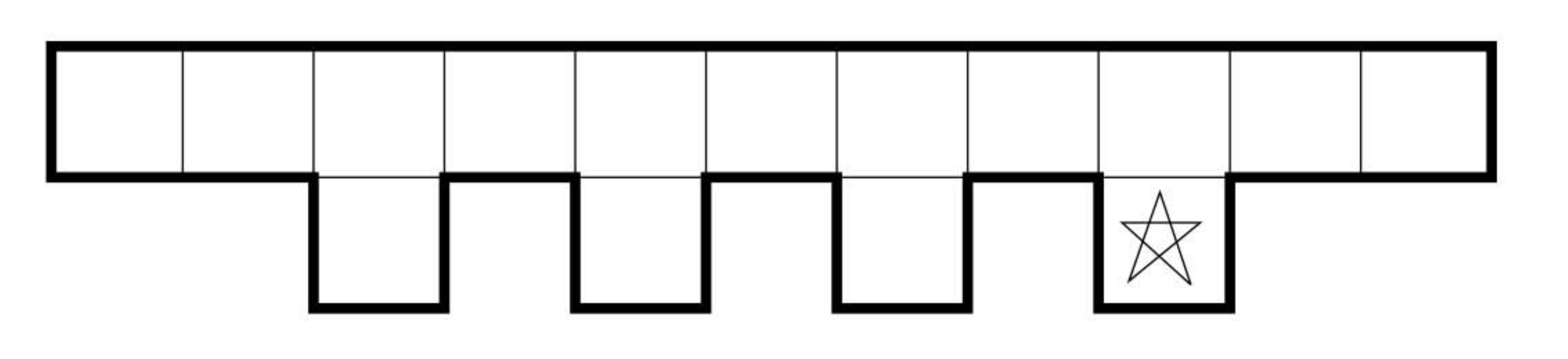}
   \caption{Hallway}
   \label{fig:hallway}
  \end{minipage}
  \begin{minipage}[b]{0.54\linewidth}
    \centering
 \centering
        \includegraphics[width=6cm]{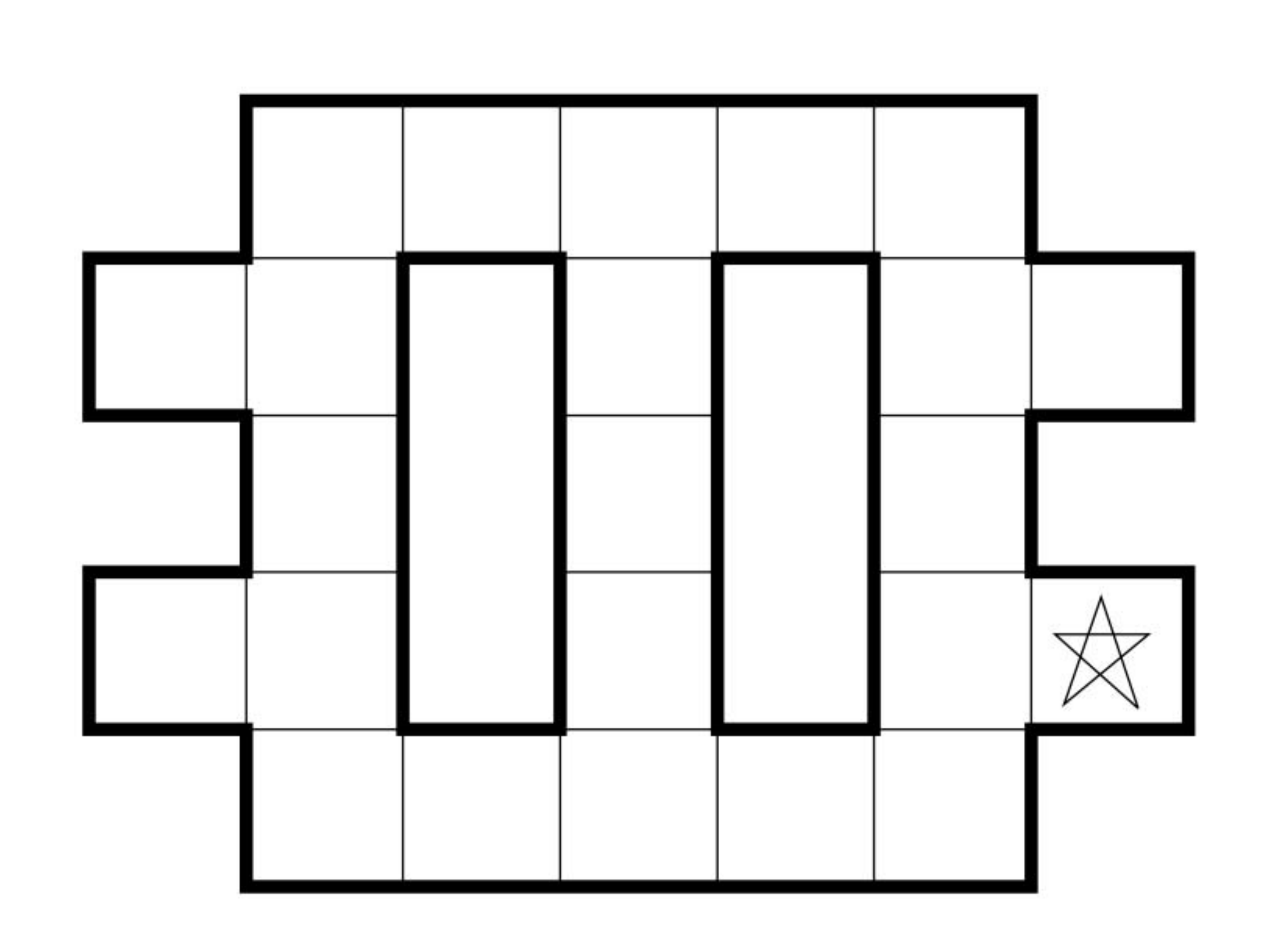}
   \caption{Hallway 2}
   \label{fig:hallway2}
  \end{minipage}
\end{figure}

\begin{figure}
\begin{center}
\includegraphics[width=5cm]{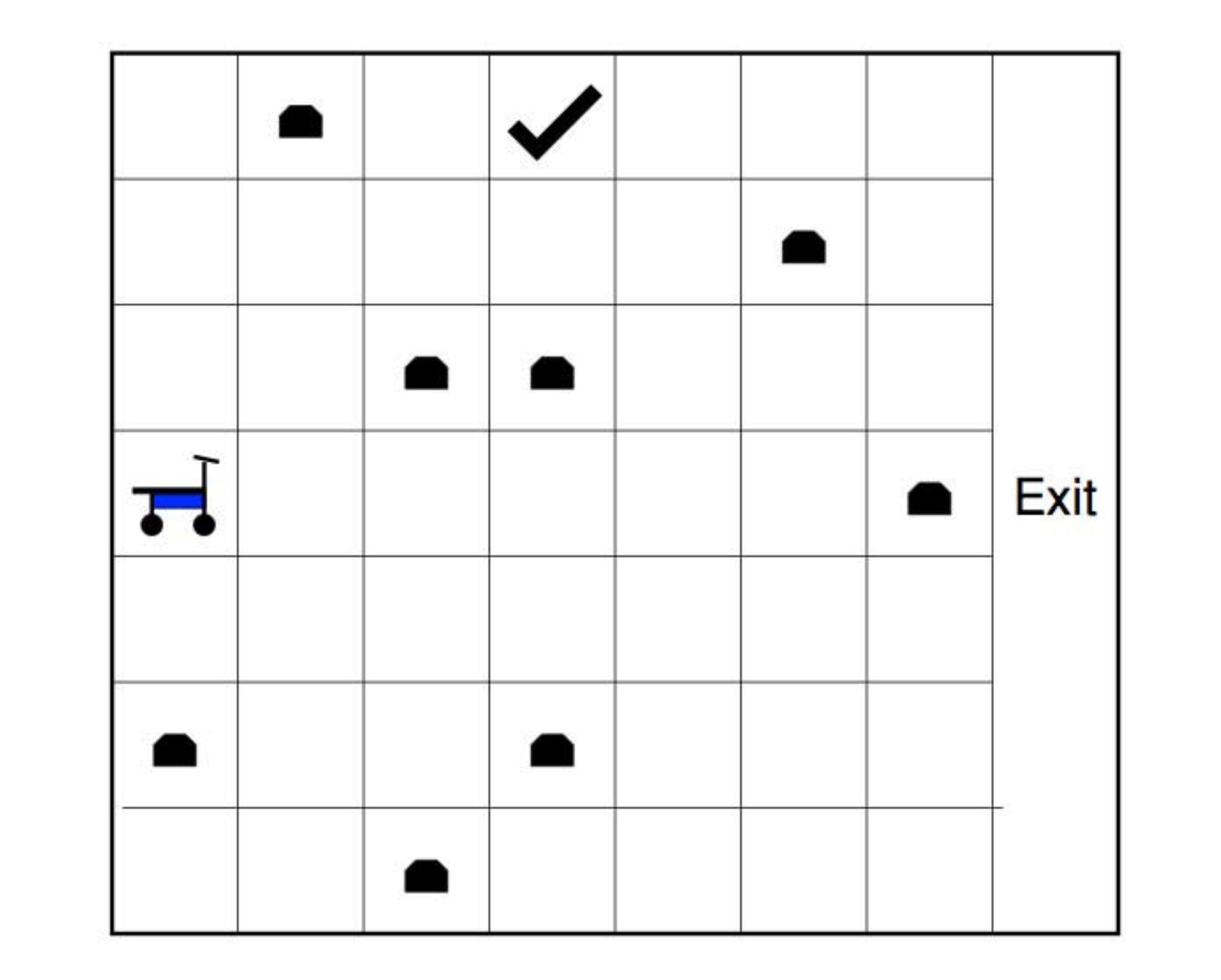}
\caption{RockSample[7,8]}
\label{fig:rockSample}
\end{center}
\end{figure}

\item \emph{RockSample.}
The RockSample problem introduced in~\cite{SS04} and used later in~\cite{BG09} is a scalable problem that models rover science exploration (Figure~\ref{fig:rockSample}). The rover can achieve various costs by sampling rocks in the immediate area, and by continuing its traverse (reaching the exit at the right side of the map). The positions of the rover and the rocks are known, but only some of the rocks have scientific value; we will call these rocks good. Sampling a rock is expensive, so the rover is equipped with a noisy long-range sensor that it can use to help determine whether a rock is good before choosing whether to approach and sample it.
An instance of RockSample with map size $n \times n$ and $k$ rocks is described as RockSample[n,k]. The POMDP model of RockSample[n,k] is as follows. The state space is the cross product of $k+1$ features: 
$\pos = \{(1, 1), (1, 2), . . . , (n, n)\}$, and $k$ binary features $\rockType_i = \{\goodObs, \badObs\}$ that indicate which of the rocks are good. There is an additional terminal state, reached when the rover moves off the right-hand edge of the map. The rover can select from $k + 5$ actions: $\{N, S, E, W, \sample, \checkAction_1 , . . ., \checkAction_k\}$. The first four are deterministic single-step motion actions. The $\sample$ action samples the rock at the rover's current location. If the rock is good, the rover receives a small cost of $1$ and the rock becomes bad (indicating that nothing more can be gained by sampling it). If the rock is bad, it receives a higher cost of $50$. The cost of performing a measurement induces a cost of $1$, attempt to move outside of the map has a cost of $100$. All other moves have a cost of $50$.
Each $\checkAction_i$ action applies the rover's long-range sensor to rock $i$, returning a noisy observation from ${Good, Bad}$. The noise in the long-range sensor reading is determined by the efficiency $\nu$, which decreases exponentially as a function of Euclidean distance from the target. At $\nu = 1$, the sensor always returns the correct value. At $\nu = 0$, it has a 50/50 chance of returning $\goodObs$ or $\badObs$. At intermediate values, these behaviors are combined linearly. The initial belief is that every rock has equal probability of being Good or Bad. All the problems have $2$ observations, and RockSample[4,4] has $257$ states, RockSample[5,5] has $801$ states,
RockSample[5,7] has 3201 states, and RockSample[7,8] has 12545 states.
\end{enumerate}

\smallskip\noindent{\textbf{Acknowledgments.}}
We thank Blai Bonet for helping us with RTDP-Bel.

\bibliographystyle{plain}
\bibliography{diss}

\begin{thebibliography}{10}

\bibitem{Billingsley}
P.~Billingsley, editor.
\newblock {\em Probability and {M}easure}.
\newblock Wiley-Interscience, 1995.

\bibitem{BG09}
B.~Bonet and H.~Geffner.
\newblock Solving {POMDPs: RTDP-B}el vs. point-based algorithms.
\newblock In {\em IJCAI}, pages 1641--1646, 2009.

\bibitem{CK13}
C.C.P Carvalho and F.~Teichteil-K\"{o}nigsbuch.
\newblock {Properly Acting under Partial Observability with Action Feasibility
  Constraints}.
\newblock volume 8188 of {\em Lecture Notes in Computer Science}, pages
  145--161. Springer, 2013.

\bibitem{POMDPSolve}
A.~Cassandra.
\newblock Pomdp-solve [software, version 5.3].
\newblock \url{http://www.pomdp.org/}, 2005.

\bibitem{cassandra1998exact}
A.R. Cassandra.
\newblock {\em Exact and approximate algorithms for partially observable Markov
  decision processes}.
\newblock Brown University, 1998.

\bibitem{CH11}
K.~Chatterjee and M.~Henzinger.
\newblock Faster and dynamic algorithms for maximal end-component decomposition
  and related graph problems in probabilistic verification.
\newblock In {\em SODA}. ACM-SIAM, 2011.

\bibitem{CY95}
C.~Courcoubetis and M.~Yannakakis.
\newblock The complexity of probabilistic verification.
\newblock {\em Journal of the ACM}, 42(4):857--907, 1995.

\bibitem{IM-Book}
K.~Culik and J.~Kari.
\newblock Digital images and formal languages.
\newblock {\em Handbook of formal languages}, pages 599--616, 1997.

\bibitem{Bio-Book}
R.~Durbin, S.~Eddy, A.~Krogh, and G.~Mitchison.
\newblock {\em Biological sequence analysis: probabilistic models of proteins
  and nucleic acids}.
\newblock Cambridge Univ. Press, 1998.

\bibitem{D00}
A.~Dutech.
\newblock Solving {POMDPs} using selected past events.
\newblock In {\em ECAI}, pages 281--285, 2000.

\bibitem{FV97}
J.~Filar and K.~Vrieze.
\newblock {\em Competitive {Markov} Decision Processes}.
\newblock Springer-Verlag, 1997.

\bibitem{Howard}
H.~Howard.
\newblock {\em Dynamic Programming and {Markov} Processes}.
\newblock MIT Press, 1960.

\bibitem{kaelbling1998planning}
L.~P. Kaelbling, M.~L. Littman, and A.~R. Cassandra.
\newblock Planning and acting in partially observable stochastic domains.
\newblock {\em Artificial intelligence}, 101(1):99--134, 1998.

\bibitem{LearningSurvey}
L.~P. Kaelbling, M.~L. Littman, and A.~W. Moore.
\newblock Reinforcement learning: A survey.
\newblock {\em J. of Artif. Intell. Research}, 4:237--285, 1996.

\bibitem{Kemeny}
J.G. Kemeny, J.L. Snell, and A.W. Knapp.
\newblock {\em Denumerable {Markov} Chains}.
\newblock D. Van Nostrand Company, 1966.

\bibitem{KMWG11}
A.~Kolobov, Mausam, D.S. Weld, and H.~Geffner.
\newblock {Heuristic search for generalized stochastic shortest path MDPs}.
\newblock In {\em ICAPS}, 2011.

\bibitem{KGFP09}
H.~Kress-Gazit, G.~E. Fainekos, and G.~J. Pappas.
\newblock Temporal-logic-based reactive mission and motion planning.
\newblock {\em IEEE Transactions on Robotics}, 25(6):1370--1381, 2009.

\bibitem{KHL08}
H.~Kurniawati, D.~Hsu, and W.S. Lee.
\newblock {SARSOP}: Efficient point-based {POMDP} planning by approximating
  optimally reachable belief spaces.
\newblock In {\em Robotics: Science and Systems}, pages 65--72, 2008.

\bibitem{LittmanThesis}
M.~L. Littman.
\newblock {\em Algorithms for Sequential Decision Making}.
\newblock PhD thesis, Brown University, 1996.

\bibitem{LCK95}
M.~L. Littman, A.~R. Cassandra, and L.~P Kaelbling.
\newblock Learning policies for partially observable environments: Scaling up.
\newblock In {\em ICML}, pages 362--370, 1995.

\bibitem{C92}
R.~A. McCallum.
\newblock First results with utile distinction memory for reinforcement
  learning.
\newblock 1992.

\bibitem{MB05}
P.~McCracken and M.~H. Bowling.
\newblock Online discovery and learning of predictive state representations.
\newblock In {\em NIPS}, 2005.

\bibitem{Mohri97}
M.~Mohri.
\newblock Finite-state transducers in language and speech processing.
\newblock {\em Computational Linguistics}, 23(2):269--311, 1997.

\bibitem{PT87}
C.~H. Papadimitriou and J.~N. Tsitsiklis.
\newblock The complexity of {M}arkov decision processes.
\newblock {\em Mathematics of Operations Research}, 12:441--450, 1987.

\bibitem{PR95}
R.~Parr and S.~J. Russell.
\newblock Approximating optimal policies for partially observable stochastic
  domains.
\newblock In {\em IJCAI}, pages 1088--1095, 1995.

\bibitem{PazBook}
A.~Paz.
\newblock {\em Introduction to probabilistic automata (Computer science and
  applied mathematics)}.
\newblock Academic Press, 1971.

\bibitem{PGT03}
J.~Pineau, G.~Gordon, S.~Thrun, et~al.
\newblock Point-based value iteration: An anytime algorithm for {POMDPs}.
\newblock In {\em IJCAI}, volume~3, pages 1025--1032, 2003.

\bibitem{Puterman}
M.~L. Puterman.
\newblock {\em {Markov} Decision Processes}.
\newblock John Wiley and Sons, 1994.

\bibitem{Rabin63}
M.O. Rabin.
\newblock Probabilistic automata.
\newblock {\em Information and Control}, 6:230--245, 1963.

\bibitem{Rbook95}
S.~J. Russell, P.~Norvig, J.~F. Canny, J.~M. Malik, and D.~D. Edwards.
\newblock {\em Artificial intelligence: a modern approach}, volume~74.
\newblock Prentice hall Englewood Cliffs, 1995.

\bibitem{SS04}
T.~Smith and R.~Simmons.
\newblock Heuristic search value iteration for {POMDPs}.
\newblock In {\em Proceedings of the 20th conference on Uncertainty in
  artificial intelligence}, pages 520--527. AUAI Press, 2004.

\bibitem{sondik}
E.~J. Sondik.
\newblock {\em The Optimal Control of Partially Observable Markov Processes.}
\newblock Stanford University, 1971.

\bibitem{S04}
M.T.J. Spaan.
\newblock A point-based {POMDP} algorithm for robot planning.
\newblock In {\em Robotics and Automation, 2004. Proceedings. ICRA'04. 2004
  IEEE International Conference on}, volume~3, pages 2399--2404. IEEE, 2004.

\end{thebibliography}

\end{document}